\documentclass[nohyperref]{article}

\usepackage{microtype}
\usepackage{graphicx}
\usepackage{subfigure}
\usepackage{booktabs} 

\usepackage{hyperref}

\usepackage{xcolor}
\definecolor{def1}{rgb}{0.12, 0.47, 0.705}
\definecolor{def2}{rgb}{1.0, 0.5, 0.055}
\definecolor{def3}{rgb}{0.17, 0.63, 0.17}
\definecolor{def4}{rgb}{0.84, 0.15, 0.16}
\definecolor{def5}{rgb}{0.58, 0.4, 0.74}
\definecolor{def6}{rgb}{0.74, 0.74, 0.13}
\DeclareRobustCommand{\legendsquare}[1]{%
  \textcolor{#1}{\rule{1.2ex}{1.2ex}}%
}
\usepackage[accepted]{icml2022}

\usepackage{amsmath}
\usepackage{amssymb}
\usepackage{mathtools}
\usepackage{amsthm}

\usepackage[utf8]{inputenc} 
\usepackage[T1]{fontenc}    
\usepackage{hyperref}
\hypersetup{hidelinks}
\usepackage{url}            
\usepackage{booktabs}       
\usepackage{amsfonts}       
\usepackage{nicefrac}       
\usepackage{microtype}      
\usepackage{xcolor}         
\usepackage{tabularx}
\usepackage{caption}
\usepackage[export]{adjustbox}
\usepackage{dashbox}
\usepackage{blindtext}
\usepackage{graphicx}
\usepackage{amsmath}
\usepackage{amsthm}
\DeclareMathOperator*{\argmax}{arg\,max}
\DeclareMathOperator*{\argmin}{arg\,min}
\usepackage{stackengine}
\usepackage{graphbox}
\usepackage{multicol}
\usepackage{float}
\usepackage{centernot}
\newsavebox\CBox

\usepackage{tikz}
\usetikzlibrary{shapes,arrows,positioning}

\newcommand{\norm}[1]{\left\lVert#1\right\rVert}
\newtheorem{definition}{Definition}
\newtheorem{example}{Example}
\newtheorem{theorem}{Theorem}
\newtheorem{proposition}{Proposition}
\newtheorem{corollary}{Corollary}

\newcommand\inner[1]{\langle #1 \rangle}
\usepackage{makecell}
\usepackage{centernot}
\usepackage{makecell}

\usepackage{bbm}
\usepackage{enumitem}   

\usepackage[capitalize,noabbrev]{cleveref}

\usepackage[textsize=tiny]{todonotes}

\icmltitlerunning{On the Adversarial Robustness of Causal Algorithmic Recourse}

\DeclareMathOperator{\sign}{sign}
\usepackage{multirow}

\begin{document}

\twocolumn[
\icmltitle{On the Adversarial Robustness of Causal Algorithmic Recourse}



\icmlsetsymbol{equal}{*}

\begin{icmlauthorlist}
\icmlauthor{Ricardo Dominguez-Olmedo}{mpi,tue}
\icmlauthor{Amir-Hossein Karimi}{mpi,eth}
\icmlauthor{Bernhard Sch\"olkopf}{mpi}
\end{icmlauthorlist}

\icmlaffiliation{tue}{University of T\"ubingen, Germany}
\icmlaffiliation{mpi}{Max Planck Institute for Intelligent Systems, T\"ubingen, Germany}
\icmlaffiliation{eth}{ETH Z\"urich, Switzerland}

\icmlcorrespondingauthor{Ricardo Dominguez-Olmedo}{ricardo.olmedo@tuebingen.mpg.de}

\icmlkeywords{Machine Learning, ICML}

\vskip 0.3in
]



\printAffiliationsAndNotice{}  

\begin{abstract}
Algorithmic recourse seeks to provide actionable recommendations for individuals to overcome unfavorable classification outcomes from automated decision-making systems. Recourse recommendations should ideally be robust to reasonably small uncertainty in the features of the individual seeking recourse. In this work, we formulate the adversarially robust recourse problem and show that recourse methods that offer minimally costly recourse fail to be robust. We then present methods for generating adversarially robust recourse for linear and for differentiable classifiers. Finally, we show that regularizing the decision-making classifier to behave locally linearly and to rely more strongly on actionable features facilitates the existence of adversarially robust recourse.
\end{abstract}

\section{Introduction}

Machine learning (ML) classifiers are increasingly being used for consequential decision-making in sensitive domains such as criminal justice and finance (e.g., granting pretrial bail or loan approval). The need to preserve human agency despite the rise in automated decisions faced by individuals has motivated the study of algorithmic recourse, which aims to empower individuals by providing them with actionable recommendations to reverse unfavorable algorithmic decisions~\cite{ustun2019actionable}. Prior works have argued that for recourse to warrant trust, the decision-maker must commit to reversing an unfavorable decision upon the decision-subject fully adopting their prescribed recourse recommendations~\cite{wachter2017counterfactual, venkatasubramanian2020philosophical,karimi2020survey}. We argue that if algorithmic recourse is indeed to be treated as a contractual agreement, then recourse recommendations must be robust to plausible uncertainties arising in the recourse process. 

For instance, consider a banking institution that promises to approve the loan of an individual if they increase their savings by some given amount. Suppose that by the time the individual achieves the prescribed savings increase, the individual's weekly working hours have been slightly reduced due to unforeseen circumstances and the decision-making classifier still deems the individual likely to default on the loan. Shielding recourse against uncertainty \emph{ex-post} by nonetheless granting the loan may be detrimental to both the bank (e.g., monetary loss) and the individual (e.g., bankruptcy and inability to secure future loans), while breaking the recourse promise would negate the effort exerted by the individual and erode trust in the decision maker. We therefore argue for the necessity of ensuring that recourse recommendations are \emph{ex-ante} robust to uncertainty.  

\begin{figure}
     \begin{center}
    \scalebox{.60}{\begin{adjustbox}{clip,trim=0.35cm 0.7cm 0.7cm 0.0cm}
            {\includegraphics{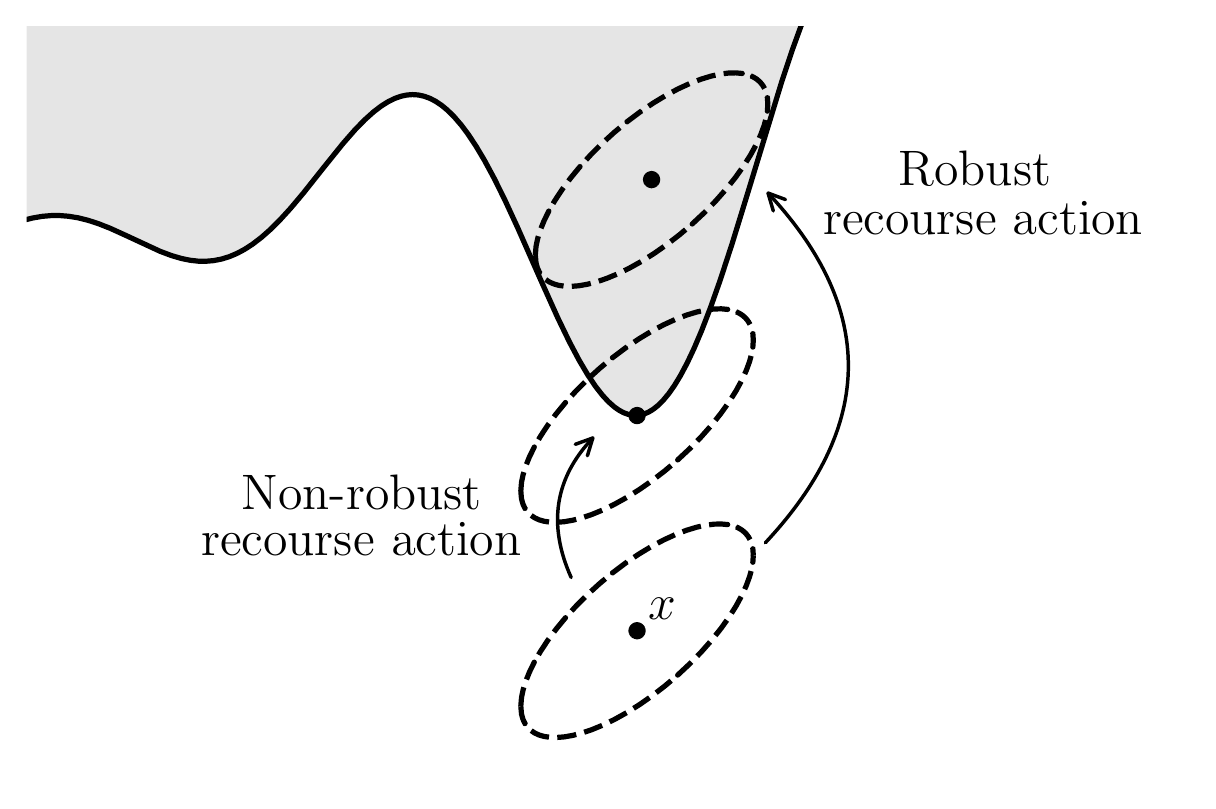}}
    \end{adjustbox}}
    \end{center}
    \caption{Adversarially robust recourse actions must lead to positive classification outcomes for \emph{all} individuals in the uncertainty set around the individual $x$ seeking recourse.}
        \label{fig:recourse}
\end{figure}

In this work, we direct our focus towards robustifying recourse recommendations against uncertainty in the features of the individual seeking recourse. Such uncertainty may arise due to the temporal nature of recourse (e.g., some features may not be static), and/or the presence of noise, adversarial manipulation and other misrepresentations or errors. We adopt a robust optimization view and propose to characterize the uncertainty around the \emph{reported} features of the individual $x$ by defining an uncertainty set $B(x)$ which we assume contains the \emph{true} features of the individual at the time recourse is offered and/or \emph{plausible} future feature values arising from the temporal nature of recourse. We then seek recourse recommendations that remain valid (i.e., lead to favorable classification outcomes) for \emph{all} plausible individuals in the uncertainty set $B(x)$, as illustrated in Figure~\ref{fig:recourse}. We refer to this notion of robustness as the \emph{adversarial robustness of recourse}.
We study the adversarial robustness of recourse from the lens of causality~\cite{pearl2009causality}. Causal recourse models recourse recommendations as causal interventions on the features of the individual seeking recourse~\cite{karimi2021algorithmic}, and therefore presents a faithful account of how the features of the individual change as the individual acts on their recourse recommendations, provided that the underlying structural causal model is known or can be approximated from observational data reasonably well~\cite{karimi2020algorithmic}.

\paragraph{Contributions}
\begin{itemize}[topsep=0pt]
    \item We formulate the adversarially robust recourse problem and show that minimum-cost recourse recommendations are provably fragile to arbitrarily small uncertainty in the features of the individual seeking recourse.
    \item We present methods for generating adversarially robust causal recourse for linear and for differentiable classifiers. We demonstrate their effectiveness on five tabular datasets, for linear and neural network classifiers.
    \item We propose a model regularizer that encourages the decision-making classifier to behave locally linearly and to rely more strongly on actionable features. We show that our proposed model regularizer facilitates the existence of adversarially robust recourse.
\end{itemize}

\section{Background and related work}

\subsection{Background on causality}\label{sec:causality}

We assume that the data-generating process of the features $\mathbf{X}=\{X_1, \ldots, X_n\}$  of individuals $x\in\mathcal{X}$ is characterized by a known \emph{structural causal model} (SCM)~\cite{pearl2009causality} $\mathcal{M}=(\mathbf{S}, P_{\mathbf{U}})$. The structural equations $\mathbf{S} = \left\{X_i := f_i\left(\mathbf{X}_{\text{pa}(i)}, \mathbf{U}_i\right)\right\}_{i=1}^n$ describe the causal relationship between any given feature $X_i$, its direct causes $\mathbf{X}_{\text{pa}(i)}$ and some exogenous variable $\mathbf{U}_i$ as a deterministic function $f_i$. The \emph{exogenous variables} $\mathbf{U}\in\mathcal{U}$, which are distributed according to some probability distribution $P_{\mathbf{U}}$, represent unobserved background factors which are responsible for the variations observed in the data. We assume that the causal graph $\mathcal{G}$ implied by the SCM, with nodes $\mathbf{X}\cup\mathbf{U}$ and edges ${\{(v, \mathbf{X}_i): v \in \mathbf{X}_{\text{pa}(i)}\cup\mathbf{U}_i, i\in[1,n]\}}$, is acyclic. The SCM $\mathcal{M}$ then implies a unique \emph{observational distribution} $p_{\mathbf{X}}$ over the features $\mathbf{X}$. Moreover, the structural equations $\mathbf{S}$ induce a mapping $\mathbb{S}: \mathcal{U} \rightarrow \mathcal{X}$ between exogenous and endogenous variables. Under the assumption that the exogenous variables are mutually independent (\emph{causal sufficiency}), if there exists some inverse mapping $\mathbb{S}^{-1}: \mathcal{X} \rightarrow \mathcal{U}$ such that $\mathbb{S}\left(\mathbb{S}^{-1}(x)\right)=x \;\; \forall x\in\mathcal{X}$, then the endogenous variables corresponding to some individual $x\in\mathcal{X}$ are uniquely identifiable by ${\mathbf{U}|x = \mathbb{S}^{-1}(x)}$.

SCMs allow for modelling and evaluating the effect of interventions on the system which the SCM models. \emph{Hard interventions} $do(\mathbf{X}_{\mathcal{I}}=\mathbf{\theta})$~\cite{pearl2009causality} fix the values of a subset $\mathcal{I}\subseteq[d]$ of features $\mathbf{X}_{\mathcal{I}}$ to some $\theta\in\mathbb{R}^{|\mathcal{I}|}$ by altering the structural equations of the intervened upon variables ${{\mathbf{S}^{do(\mathbf{X}_{\mathcal{I}}=\theta)}_{\mathcal{I}_i} = \mathbf{X}_{\mathcal{I}_i} := \mathbf{\theta}_i}}$ while preserving the rest of the structural equations ${\mathbf{S}^{do(\mathbf{X}_{\mathcal{I}}=\mathbf{\theta})}_i = \mathbf{S}_i \;\ \forall i\notin\mathcal{I}}$. Thus, hard interventions sever the causal relationship between an intervened upon variable and all of its ancestors in the causal graph. Soft interventions, on the other hand, may modify the structural equations in a more general manner~\cite{korb2004varieties}. In particular, \emph{additive interventions}~\cite{eberhardt2007interventions} perturb the features $\mathbf{X}$ by some perturbation vector $\Delta\in\mathbb{R}^n$ while preserving all causal relationships, altering the structural equations to ${\mathbf{S}^{\Delta} = \left\{X_i := f_i\left(\mathbf{X}_{\text{pa}(i)}, \mathbf{U}_i\right)+\Delta_i\right\}_{i=1}^n}$.

Moreover, SCMs imply distributions over \emph{counterfactuals}, allowing to reason about what would have happened under certain hypothetical interventions all else being equal. The counterfactual $x^{\text{CF}}$ pertaining to some observed factual individual ${x\in\mathcal{X}}$ under some hypothetical hard intervention $ do(\mathbf{X}_{\mathcal{I}}=\mathbf{\theta})$ (resp. soft intervention $\Delta$) can be computed by first determining the exogenous variables $\mathbf{U}|x = \mathbb{S}^{-1}\left(x\right)$ corresponding to the individual $x$, and then applying the interventional mapping $\mathbb{S}^{do(\mathbf{X}_{\mathcal{I}}=\mathbf{\theta})}$ (resp. $\mathbb{S}^{\Delta}$) from endogenous to exogenous variables~\cite{pearl2009causality}. For notational convenience, we denote such mapping as $x^{\text{CF}} = \mathbb{CF}\left(x, do(\mathbf{X}_{\mathcal{I}}=\mathbf{\theta})\right):=\mathbb{S}^{do(\mathbf{X}_{\mathcal{I}}=\mathbf{\theta})}\left(\mathbb{S}^{-1}\left(x\right)\right)$ (resp. $x^{\text{CF}} = \mathbb{CF}\left(x, \Delta\right):=\mathbb{S}^{\Delta}\left(\mathbb{S}^{-1}\left(x\right)\right)$). We use the notation $x^{\text{CF}}=\mathbb{CF}\left(x, do(\mathbf{X}_{\mathcal{I}}=\mathbf{\theta}); \mathcal{M}\right)$ to highlight that the counterfactual $x^{\text{CF}}$ follows from a particular SCM $\mathcal{M}$.

\subsection{The causal recourse problem}\label{causrec}

Consider the setting where a classifier $h: \mathcal{X} \rightarrow \{0,1\}$ is used to assign either favorable or unfavorable outcomes to individuals $x\in\mathcal{X}$ (e.g., loan approval). 
We adopt the causal view of recourse introduced by~\citet{karimi2021algorithmic} and model recourse recommendations as hard interventions on the features of the individual seeking recourse. We consider interventions of the form ${a(x)=do(\mathbf{X}_{\mathcal{I}}=x_{\mathcal{I}}+\mathbf{\theta})}$, where ${\theta\in\mathbb{R}^{|\mathcal{I}|}}$ is the prescribed change to a subset of the features of the individual $x$. We consider this additive form, rather than ${a=do(\mathbf{X}_{\mathcal{I}}=\mathbf{\theta})}$ as \citet{karimi2021algorithmic}, to explicitly allow for uncertainty in the individual $x$ to propagate to the action $a(x)$. Otherwise, hard-intervening on all features would trivially shield the counterfactual $\mathbb{CF}(x, a)$ from uncertainty in~$x$. For notational simplicity, we refer to an action $a(x)$ as simply $a$ when the pertinent $x$ is clear from context (e.g., we refer to $\mathbb{CF}(x, a(x))$ as simply $\mathbb{CF}(x, a)$).

For a recourse action $a$ to be considered \emph{valid}, the corresponding counterfactual individual must be favorably classified, that is, $h\left(\mathbb{CF}\left(x, a; \mathcal{M}\right)\right)=1$. Since certain features may be immutable (e.g., race) or bounded (e.g., age), only feasible actions should be recommended. The action feasibility set $\mathcal{F}(x)$ captures the set of feasible actions available to an individual $x$. Additionally, the cost function $c(x,a)$ models the effort required by an individual $x\in\mathcal{X}$ to implement some recourse action~$a$. Finding the least effortful (i.e., minimum-cost) recourse action for some individual $x\in\mathcal{X}$ amounts to solving the following optimization problem:
\begin{equation}
\begin{aligned}
\argmin_{a(x)=do(\mathbf{X}_{\mathcal{I}}=x_{\mathcal{I}}+\mathbf{\theta})} \quad & c(x, a)\\
\textrm{s.t.} \quad & a\in\mathcal{F}(x) \\
  &\dbox{$h$}\left(\mathbb{CF}\left(\dbox{$x$}, a; \dbox{$\mathcal{M}$}\right)\right)=1  \\
\end{aligned}
\label{eq:recourse}
\end{equation}

As highlighted in Equation~\ref{eq:recourse}, uncertainty in the features of the individual $x$, the classifier $h$, and/or the SCM $\mathcal{M}$ may affect the validity of recourse. In Appendix~\ref{sec:uncertainties}, we discuss and relate the different sources of uncertainty arising throughout the recourse process.

The non-causal recourse setting is equivalent to the causal recourse setting under the \emph{independently manipulable features} (IMF) assumption, that is, if no causal relationships exist between the features of the individual. Under such assumption, $\mathbb{CF}\left(x, do(\mathbf{X}=x+\theta)\right)=x+\theta$.

\subsection{Related work}

We now draw connections with existing literature on the robustness of recourse. Previous works have studied the problem of generating recourse actions which remain valid under uncertainty in the classifier $h$. \citet{pawelczyk2020counterfactual} show that recourse actions which place the counterfactual in regions of the feature space with large data support are more robust to competing classifiers that perform equally well. However, recourse actions with large data support may be unnecessarily costly. In contrast, we seek robust recourse actions that are also minimally costly. Another line of work has considered robustness of recourse against changes to the classifier in response to dataset shift. \citet{rawal2020can} show that recourse actions are typically not robust to such model changes, and \citet{upadhyay2021towards} aim to mitigate this issue by generating recourse with a minimax optimization procedure where the cost of recourse is minimized subject to the recourse action being valid under adversarial changes to the classifier~$h$. While we adopt a similar minimax approach to generate robust recourse, we focus on robustifying recourse against uncertainty in the individual~$x$ rather than the classifier $h$. Subsequent works propose to instead adopt a distributionally robust viewpoint~\cite{black2021consistent, bui2021counterfactual, guo2022rocoursenet}. Likewise, a natural extension of our work is to adopt a distributionally robust optimization viewpoint.

Regarding robustness of recourse against uncertainty in the SCM $\mathcal{M}$, \citet{karimi2020algorithmic} consider the setting where the underlying SCM is not know and thus must be approximated from data, and propose a method to generate recourse recommendations which have low probability of being invalid due to the misspecification of the underlying SCM. Our work is tangential to \citet{karimi2020algorithmic}.

Previous works have identified that small changes to the features of the decision-subject $x$ may result in different recourse recommendations being offered with potentially very different costs~\citep{von2020fairness, slack2021counterfactual,artelt2021evaluating}. Instead of focusing on the cost of recourse, we study whether recourse actions remain valid under uncertainty in the individual $x$. In Appendix~\ref{sec:robcost}, we discuss in more detail the relation between these two different notions of robustness. The concurrent work of  \citet{virgolin2022robust} is most similar to ours, as they study the robustness of recourse to adversarial perturbations to the individual $x$. They propose an evolutionary algorithm to generate robust recourse, which they evaluate for random forest classifiers. In contrast, we focus on generating recourse for differentiable classifiers, in particular linear and neural network models. Additionally, we consider the more general causal recourse setting, and we model feature perturbations in a causal manner. Lastly, \citet{pawelczyk2022algorithmic} study the rate at which random perturbations to the actionable features of $x$ invalidate recourse, and propose a method to generate recourse which is less likely to be invalidated. In contrast, we aim to generate recourse which is, to the extent possible, provably robust against adversarial perturbations. We additionally consider the causal setting.

Finally, \citet{ross2021learning} propose to regularize the decision-making classifier at training time to facilitate the existence of recourse. They propose a regularizer which encourages the classifier to be very sensitive (i.e., fragile) to changes to the actionable features. Such fragility of the classifier, however, can be problematic when generating adversarially robust recourse. In contrast, our proposed regularizer encourages the classifier to be more robust to adversarial examples by regularizing it to behave locally linearly~\cite{qin2019adversarial}.

\section{Counterfactual uncertainty sets}

In the adversarial robustness literature, the uncertainty surrounding an observation $x$ is often modelled by an $\epsilon$-ball of uncertainty $B(x) = \{x + \Delta | \norm{\Delta} \leq \epsilon\}$ around the observed $x$, where the norm $\norm{\cdot}$ characterizes  some relevant notion of magnitude for perturbations $\Delta$ to the observation $x$ and $\epsilon$ specifies the amount of uncertainty under consideration~\cite{madry2017towards, bertsimas2019robust}.
Intuitively, small perturbations $\Delta$ to the observation $x$ result in plausibly similar data points.
Then, the uncertainty set $B(x)$ amounts to an $\epsilon$-neighborhood of plausible data points similar to~$x$.

 \begin{figure}
      \begin{center}
    \scalebox{.60}{\begin{adjustbox}{clip,trim=0.6cm 0cm 0cm 1.6cm}
            \includegraphics{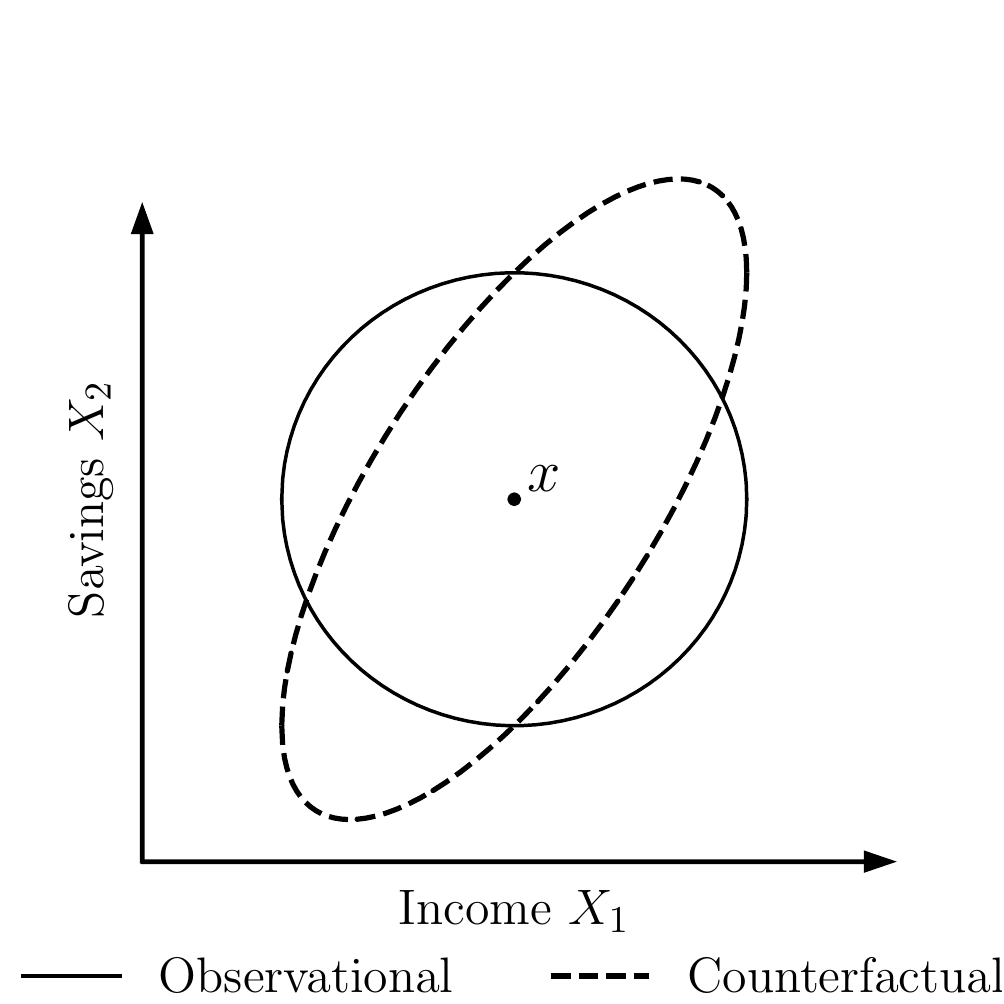}
    \end{adjustbox}}
    \end{center}
    \caption{Illustration of the observational and counterfactual neighborhoods of similar individuals for the linear SCM $X_1=U_1$ (Income), $X_2=X_1+U_2$ (Savings) under $\norm{\cdot}_2$.}
        \label{fig:similarity}
\end{figure}

From a causal perspective, an underlying assumption in the above definition of $B(x)$ is that the features are independently manipulable. We argue, however, that explicitly considering the causal relationships between features can result in potentially more informative neighborhoods of similar individuals. To account for the causal effects of the perturbations $\Delta$, we propose to model such perturbations as additive interventions on the features of the individual $x$. 

 \begin{definition}[Counterfactual neighborhood of similar individuals] For some individual $x$, SCM $\mathcal{M}$ and norm $\norm{\cdot}$, we define the counterfactual $\epsilon$-neighborhood $B(x)$ of individuals similar to the observed $x$ as the set of causal counterfactuals\footnote{As introduced in Section~\ref{sec:causality}, do not confuse with the notion of ``counterfactual examples'' from the recourse literature.} under all possible $\epsilon$-small additive interventions:
\vspace{-0.2cm}\begin{equation}
    B(x) = \left\{\mathbb{CF}\left(x, \Delta; \mathcal{M}\right) \; | \; \norm{\Delta} \leq \epsilon\right\}
\end{equation}
\end{definition}
\vspace{-0.10cm}As a motivating example, consider the SCM $\mathcal{M}$ with features $X_1=U_1$ and $X_2=X_1 + U_2$ respectively denoting the income and savings of some individual $x$. Figure~\ref{fig:similarity} illustrates the corresponding observational and counterfactual neighborhoods $B(x)$. Observe that the counterfactual neighborhood is skewed towards individuals with both higher (or lower) income and savings, since we would expect perturbations that simultaneously decrease income and increasing savings to be less likely to be observed (i.e., are not well-supported by the causal structure of the data). Therefore, we argue that counterfactual neighborhoods can be more informative, since the causal relationships between features are explicitly considered. Additionally, for non-linear SCMs, counterfactual neighborhoods $B(x)$ adapt to the local geometric structure of the data manifold (see Appendix~\ref{sec:neig}).

\section{The adversarially robust recourse problem}\label{recproblm}

We consider the problem of generating recourse actions that are robust to uncertainty in the features of the individual seeking recourse. 
We adopt a robust optimization viewpoint~\cite{ben2009robust} and model uncertainty in the individual $x$ by characterizing an uncertainty set $B(x)$ of plausible individuals around $x$. We formalize the notion of adversarial robustness of recourse as follows:

\begin{definition}[Adversarially robust recourse action] For some classifier $h$, individual $x\in\mathcal{X}$, and uncertainty set $B(x)$, a recourse action $a$ is adversarially robust if it is valid for all individuals in the uncertainty set~$B(x)$
\begin{equation}
    h\left(\mathbb{CF}\left(x', a\right)\right)=1 \quad \forall x' \in B(x)
\end{equation}
\end{definition}

We additionally define the adversarially robust recourse problem by directly incorporating the above robustness requirement as a constraint to the standard recourse problem:

\begin{definition}[Adversarially robust recourse problem] For some uncertainty set $B(x)$, the minimum-cost recourse action which is adversarially robust is given by
\label{def:robrec}
\begin{equation}
\begin{aligned}
\argmin_{a(x)=do(\mathbf{X}_{\mathcal{I}}=x_{\mathcal{I}}+\mathbf{\theta})} \; & c(x, a)\\
\textrm{s.t.} \quad & a\in\mathcal{F}(x) \\
  &h\left(\mathbb{CF}\left(x', a\right)\right)=1 \;\; \forall x' \in B(x)  \\
\end{aligned}
\end{equation}
\end{definition}

Note that any feasible solution to the above optimization problem is by definition adversarially robust.

\subsection{Recourse is fragile under mild conditions}

We show that under mild conditions on the cost function $c$, feasibility set $\mathcal{F}(x)$ and SCM $\mathcal{M}$, minimum-cost recourse actions are provably fragile to arbitrarily small uncertainty in the features of the individual seeking recourse.

\begin{theorem} Let $a^*$ be the solution to the standard recourse optimization problem defined in Equation~\ref{eq:recourse}. If
\label{th:fragile}
\begin{enumerate}[label=(\roman*),topsep=0pt,itemsep=-1ex,partopsep=1ex,parsep=1ex]
    \item The cost function $c(x, do(X_{\mathcal{I}}=x_{\mathcal{I}}+\theta))$ is strictly convex in $\theta$ with minimum $\theta=0$
    \item $do(X_{\mathcal{I}}=x_{\mathcal{I}}+\theta) \in \mathcal{F}(x)$ $\implies$ \phantom{fillfillfillfillfillfillfillfil} ${do(X_{\mathcal{I}}=x_{\mathcal{I}}+t\theta) \in \mathcal{F}(x) \quad \forall \; 0 < t< 1}$
    \item The SCM $\mathcal{M}$ is an additive noise model~\cite{pearl2009causality}.
\end{enumerate}

Then, for any ${\epsilon>0}$ there exists some plausible individual $x' \in B(x)$ in the $\epsilon$-neighbourhood of similar individuals $B(x)=\left\{\mathbb{CF}\left(x, \Delta\right) | \norm{\Delta} \leq \epsilon\right\}$  such that $a^*$ is not a valid recourse action, i.e., $h(\mathbb{CF}(x', a^*))=0$. That is, the minimum-cost recourse action $a^*$ is fragile to arbitrarily small uncertainty in the individual~$x$ seeking recourse.
\end{theorem}

Condition $(i)$ requires that larger feature changes imply strictly more effort. This condition is satisfied by weighted p-norms~\cite{karimi2020algorithmic} and percentile costs~\cite{ustun2019actionable}, the most widely used cost functions in the recourse literature. Condition $(ii)$ states that if it is feasible to change a feature by some amount, then it must also be
feasible to change that feature to a lesser degree. This condition is satisfied by box actionability constrains, commonly assumed in the recourse literature~\cite{karimi2020survey}. Lastly, condition $(iii)$ is a common modelling assumption for estimating the SCM $\mathcal{M}$ from data~\cite{karimi2020algorithmic}, and also holds in the non-causal recourse setting. 

Therefore, in the settings commonly considered by the algorithmic recourse literature, methods seeking minimum-cost recourse offer provably fragile recourse recommendations.

\subsection{On the existence of adversarially robust recourse}\label{sec:conds}

We extend the sufficient conditions for the existence of recourse derived by \citet{ustun2019actionable} to the adversarially robust causal recourse setting. 
First, we state the negative result that even if recourse exists for every individual $x \in \mathcal{X}$, robust recourse may not exist for any individual $x\in\mathcal{X}$.

\begin{proposition}\label{noexistence}
The existence of recourse for all individuals $x \in \mathcal{X}$ is not a sufficient condition for the existence of adversarially robust recourse for any $x \in \mathcal{X}$, even under the strong assumption that all features are actionable.
\end{proposition}

Intuitively, for robust recourse to exist the decision-making classifier must be minimally robust in the sense that there must exist at least one individual which is robustly classified. In this sense, robustness of prediction is necessary (but not sufficient) for the existence of robust recourse. Under mild robustness conditions on the classifier, all features being actionable is sufficient for the existence of robust recourse.

\begin{proposition}\label{existence1}
For an SCM $\mathcal{M}$ with linear structural equations, if all features are actionable and there exists some robustly classified individual $x^+ \in \mathcal{X}$ such that  $h(x') = 1 \; \forall x'\in B(x^+)$, then there exists an adversarially robust recourse action for every individual $x \in \mathcal{X}$.
\end{proposition}

Intuitively, one can then require all negatively classified individuals to act such that their features resemble $x^+$. Note that the above statement holds only for linear SCMs, since for non-linear SCMs the shape of the uncertainty set $B(x)$ differs between individuals $x$, and thus it is not possible to state a single robustness requirement for $h$ at~$x^+$.

By further assuming that the classifier $h$ is linear, it is possible to relax the condition of all features being actionable to a single feature being actionable and unbounded:

\begin{proposition}\label{existence2}
For a linear classifier $h$ and linear SCM $\mathcal{M}$, under mild conditions on the weights of the classifier described in Appendix~\ref{sec:proofex2}, if there exists a feature $\mathbf{X}_j$ that is actionable and unbounded, then there exists an adversarially robust recourse action for every $x \in \mathcal{X}$. 
\end{proposition}

\section{Generating adversarially robust recourse}

\subsection{The linear case}

For a linear classifier ${h(x) = \inner{w, x} \geq b}$ and linear SCM, we show that generating robust recourse for the classifier $h$ is equivalent to generating standard recourse for a modified linear classifier ${h'(x) = \inner{w, x} \geq b'}$ whose ``acceptance threshold'' is sufficiently increased (i.e., $b'\geq b$).

\begin{proposition} \label{linear}
Let $h(x) = \inner{w, x} \geq b$ be a linear classifier, $\mathcal{M}$ an SCM with linear structural equations, and $B(x) = \left\{\mathbb{CF}\left(x,\Delta\right) \; | \; \norm{\Delta} \leq \epsilon\right\}$ an uncertainty set of plausible individuals. Then, an action $a(x)=do(\mathbf{X}_{\mathcal{I}}=x_{\mathcal{I}}+\theta)$ is an adversarially robust recourse action if and only $a$ is a valid recourse action for the following modified classifier:
\begin{equation}
    h'(x) = \inner{w, x} \geq b + \norm{J^T_{\mathbb{S}^\mathcal{I}}w}^*\epsilon
    \label{eq:p13}
\end{equation}
where $\norm{\cdot}^*$ denotes the dual norm of $\norm{\cdot}$ and $J_{\mathbb{S}^\mathcal{I}}$ denotes the Jacobian of the interventional mapping resulting from hard-intervening on a subset of features $\mathbf{X}_{\mathcal{I}}$.
\end{proposition}
\begin{corollary} For any given individual $x\in\mathcal{X}$, the minimum-cost adversarially robust recourse action for the original classifier $h$ is equivalent to the minimum-cost standard recourse action for the modified classifier~$h'$.
\end{corollary} 

Therefore, in the linear setting, per Corollary~1 any method for generating standard recourse can be readily used to generate adversarially robust recourse by simply considering the modified classifier $h'$. In such cases, adversarial robustness can be straightforwardly embedded within methods seeking to promote other desiderata, such as large data-support~\cite{joshi2019towards} or fairness constraints~\cite{gupta2019equalizing, von2020fairness}.

\subsection{The differentiable case}

We now consider the setting where the classifier $h$ and SCM $\mathcal{M}$ are differentiable. First, we derive an unconstrained penalty problem that is equivalent to the adversarially robust recourse problem. We then discuss how to efficiently solve said unconstrained optimization problem.

\begin{proposition} \label{differentiable}
Let $h(x) = \tilde h(x) \geq b$ for differentiable $\tilde h: \mathcal{X} \rightarrow \left[0, 1\right]$. The adversarially robust recourse problem is equivalent to the following unconstrained problem:
\begin{equation*}\label{eq:maxminmax}
    \hspace*{-0.4cm}\min_{a\in\mathcal{F}(x)} \max_{\lambda \geq 0} \; c(x,a) + \lambda  \left(\log b + \max_{x'\in B(x)} \ell \left(\tilde h\left(\mathbb{CF}\left(x', a\right)\right), \mathbf{1}\right)\right)
\end{equation*}
where $\ell$ is the binary cross-entropy loss. 
\end{proposition}

\begin{algorithm}[t]
\caption{Generate adversarially robust recourse for a differentiable classifier $h$ and differentiable SCM $\mathcal{M}$.}
\label{alg:adv}
\begin{algorithmic}[1]
\REQUIRE Factual individual $x$, uncertainty set $B(x)$, subset $\mathcal{I}$ of intervened-upon features $x_{\mathcal{I}}$ , $\lambda > 0$, $\gamma > 1$, $\theta = 0$
\WHILE{$N \leq N_{\max}$}
\WHILE{not converged}
\STATE $a(x) \gets do(X_{\mathcal{I}}=x_{\mathcal{I}}+\theta)$
\STATE $x^* \gets \argmax_{x'\in B(x)} \ell\left(\tilde h(\mathbb{CF}(x', a)), \mathbf{1}\right)$
\IF{$h\left(\mathbb{CF}\left(x^*, a\right)\right)=1$}
    \STATE \textbf{return} $\theta$
\ENDIF
\STATE $g \gets \nabla_{\theta}\left(c(x, a) + \lambda \; \ell \left(\tilde h(\mathbb{CF}(x^*, a)), \mathbf{1}\right)\right)$
\STATE $\theta \gets \textrm{Proj}_{\mathcal{F}(x)}(\theta - \alpha g)$
\ENDWHILE
\STATE $\lambda \gets \gamma \lambda$
\ENDWHILE
\end{algorithmic}
\end{algorithm}
\setlength{\textfloatsep}{12pt}

To solve the outer minimax optimization problem, we adopt the causal recourse approach of \citet{karimi2020algorithmic} and use projected gradient descent over the recourse action $a$ and feasibility set $\mathcal{F}(x)$, while iteratively increasing $\lambda$ to place growing emphasis in crossing the classifier's decision boundary, as described in Algorithm~\ref{alg:adv}. Note that if there is no uncertainty in the features of the individual~$x$, that is $B(x)=\{x\}$, then the optimization procedure is precisely that of \citet{karimi2020algorithmic} and \citet{wachter2017counterfactual} for the causal and non-casual recourse settings respectively, since the solution to the inner maximization is trivially $x^*=x$. 

We solve the inner maximization in Proposition~\ref{differentiable} using projected gradient ascent over the uncertainty set~$B(x)$. Since this inner maximization problem is in general non-convex, the local maximum $x^*$ found with gradient ascent may not be the global maximum in $B(x)$. If $x^*$ is not the global maximum, then in Algorithm~1 the exit condition $h\left(\mathbb{CF}\left(x^*, a\right)\right)=1$ does not imply that ${h\left(\mathbb{CF}\left(x', a\right)\right)=1} \; \forall x' \in B(x)$. Thus, in general it is not possible to guarantee that the recourse actions returned by the proposed algorithm are adversarially robust. However, as discussed in the experiments section, we empirically find that, for sufficiently small uncertainty $\epsilon$, the proposed algorithm is effective in generating robust recourse actions.

\section{Actionability regularization}

While the robustness of recourse is desirable for both the decision-maker and the decision-subject, the burden of immunizing recourse against uncertainty falls solely on the decision-subject: robust recourse is more effortful (minimum-cost recourse is provably fragile) and may not even exist (Proposition~\ref{noexistence}). To regulate the burden of robustness between the decision-maker and the decision-subject, we propose to regularize the decision-making classifier at training time with the aim of 1) facilitating the existence of robust recourse, and 2) reducing the extra cost of seeking robust recourse. Regularizing the classifier to promote such desiderata may come at a cost in predictive performance, thus shifting part of the burden of robust recourse from the decision-subject to the decision-maker. We propose to regularize the decision-making classifier to behave locally linearly and to rely more strongly on actionable features. 

\subsection{Theoretical motivation}

To motivate our proposed regularizer, we refer to the sufficient conditions for the existence of adversarially robust recourse presented in Section~\ref{sec:conds}. Per Proposition~\ref{existence1} and Proposition~\ref{existence2}, we hypothesize, respectively, that using only actionable features for classification and using a linear classifier facilitates the existence of robust recourse. However, often times most available features are unactionable, and thus classifiers trained only with actionable features will exhibit poor predictive performance. Similarly,  the predictive performance of linear classifiers is generally inferior to that of nonlinear classifiers. As a middle ground, we propose to use all available features to train a potentially non-linear classifier, but regularize the classifier to rely more strongly on the actionable features and to locally behave linearly.

Such choice of regularization is additionally well-motivated from the viewpoint of reducing the extra cost of robustifying recourse, for which we derive the following upper bound:

\begin{proposition}\label{th:cost}
Let $h(x)$ be a linear classifier $\inner{w, x} \geq b$, $x \in \mathcal{X}$ a negatively classified individual for which there exists some recourse action $a(x)=do(\mathbf{X}_{\mathcal{I}}=x_{\mathcal{I}}+\theta)$, and $B(x)=\left\{\mathbb{CF}(x, \Delta) \; | \; \norm{\Delta} \leq \epsilon\right\}$ the uncertainty set. If the features $X_{\mathcal{I}}$ are unbounded and independently manipulable, and the cost function is subadditive, then there exists some adversarially robust recourse action $a'$ such that the extra cost of robustifying recourse is upper bounded by
\begin{equation}
  \label{eq:beta}
  \dfrac{c(x, a') - c(x, a)}{c(x, a)} \leq \dfrac{\norm{m_{\mathcal{A}}\odot w}^*+\norm{m_{\tilde{\mathcal{A}}}\odot w}^*}{\inner{m_{\mathcal{A}}\odot w, \theta}} \epsilon
\end{equation}
where $m_{\mathcal{A}} \in [0,1]^n$ (resp. $m_{\tilde{\mathcal{A}}}$) is the mask vector for the set of actionable features $\mathcal{A}$ (resp. unactionable features~$\tilde{\mathcal{A}}$), and $\norm{\cdot}^*$ is the dual norm of $\norm{\cdot}$. 
\end{proposition}

Observe that the upper bound presented above is reduced if the unactionable features are less discriminative (i.e., if $\norm{m_{\tilde{\mathcal{A}}}\odot w}^*$ is small). By additionally regularizing non-linear classifiers to behave locally linearly, we ensure that this upper bounds approximately holds locally. In Appendix \ref{bound}, we extend the above upper bound to the causal setting. 

\subsection{The Actionable Locally Linear Regularizer}

To formalize our proposed regularizer, we draw inspiration from local linearity regularization~\cite{qin2019adversarial}, a popular regularization technique from the adversarial robustness literature. 
For differentiable classifiers $h(x)$, we propose the  Actionable Locally Linear Regularizer (ALLR):
\begin{equation}\label{eq:allr}
\begin{split}
\mathcal{R}(x) &= \mu_1  \max_{\norm{\delta} \leq \epsilon}| h(x+\delta) - \inner{\delta, \nabla_x h(x)} -  h(x)| \\ & + \mu_2 \norm{m_{\tilde{\mathcal{A}}}\odot\nabla_x h(x)}
\end{split}
\end{equation}

The first term in the ALLR regularizer encourages the classifier $h$ to behave linearly near the observed data, while the second term encourages unactionable features to not be very discriminative. The hyperparameters $\mu_1,\mu_2\in\mathbb{R}$ determine the strength of regularization. The classifier is then trained using regularized risk minimization:
\begin{equation}
    \min_{\psi} \mathbb{E}_{(x,y)\sim p(x,y)}\left[\ell(h_{\psi}(x), y) + \mathcal{R}(x)\right]
\end{equation}
where $\ell$ is the binary cross-entropy loss, $p(x,y)$ is the training data distribution, and $\psi$ are the weights of the classifier.

\section{Experiments and results}

First, we empirically validate the effectiveness of the proposed methods in generating adversarially robust recourse. Secondly, we empirically show that regularizing the decision-making classifier with the proposed ALLR regularizer facilitates finding adversarially robust recourse actions. We open source our implementations and experiments\footnote{\href{https://github.com/ricardodominguez/adversariallyrobustrecourse}{github.com/RicardoDominguez/AdversariallyRobustRecourse}}. 

We consider four real-world datasets and one semi-synthetic dataset. For the causal recourse setting, we consider the COMPAS recidivism dataset~\cite{larson2016we} and the Adult demographic dataset~\cite{kohavi1996adult}, for which we adopt the causal graphs assumed in \citet{nabi2018fair}. We fit the structural equations using linear models and neural nets for the linear and the non-linear case, respectively. We additionally consider one semi-synthetic SCM introduced by \citet{karimi2020algorithmic}, which is inspired in a loan approval setting. For the non-causal recourse setting, we consider the South German Credit dataset~\cite{groemping2019south}, as well as a recidivism dataset~\cite{schmidt1988predicting} from North Carolina which we refer to as Bail. In Appendix~\ref{sec:feats} we list the actionability constraints considered.

For all datasets, we treat actionable categorical variables as real-valued, and we standardize all real-valued features. We use as the cost function the $\ell_1$ norm, that is $c(x, a) = \norm{\theta}_1$ for $a(x) = do(X_{\mathcal{I}}=x_{\mathcal{I}}+\theta)$. We consider two types of classifiers: linear models trained with logistic regression (LR), and neural network (NN) models. We define the uncertainty set $B(x)$ with respect to the 2-norm. Since features are standardized, robustifying against $\epsilon$ uncertainty is equivalent to guarding against perturbations at least $\epsilon$ times the standard deviation of any given feature. For instance, in the Adult dataset the standard deviation of the feature \emph{age} is 13.6 years. Thus, for $\epsilon=0.1$ uncertainty robust actions should remain valid even if the age of the individual seeking recourse changes by $\pm~1.36$ years. In Appendix~\ref{sec:stds} we list the standard deviation of the considered features.

 \begin{figure}
      \begin{center}
    \scalebox{0.85}{\begin{adjustbox}{clip,trim=0.25cm 0.30cm 0.30cm 0.25cm}
            \includegraphics{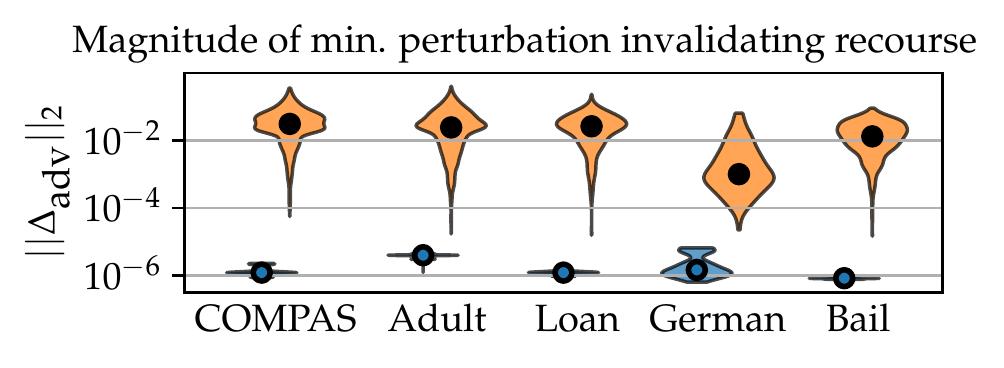}
    \end{adjustbox}}
    \end{center}
    \caption{Fragility of standard recourse. Very small feature perturbations can invalidate recourse, particularly for linear classifiers. Legend: \legendsquare{def1} LR classifier~~\legendsquare{def2} NN classifier.} 
        \label{fig:pertb}
\end{figure}

\subsection{Minimum-cost recourse is fragile}

\begin{figure*}
      \begin{center}
    \scalebox{0.85}{\begin{adjustbox}{clip,trim=0.3cm 0.35cm 0.4cm 0.0cm}
            \includegraphics{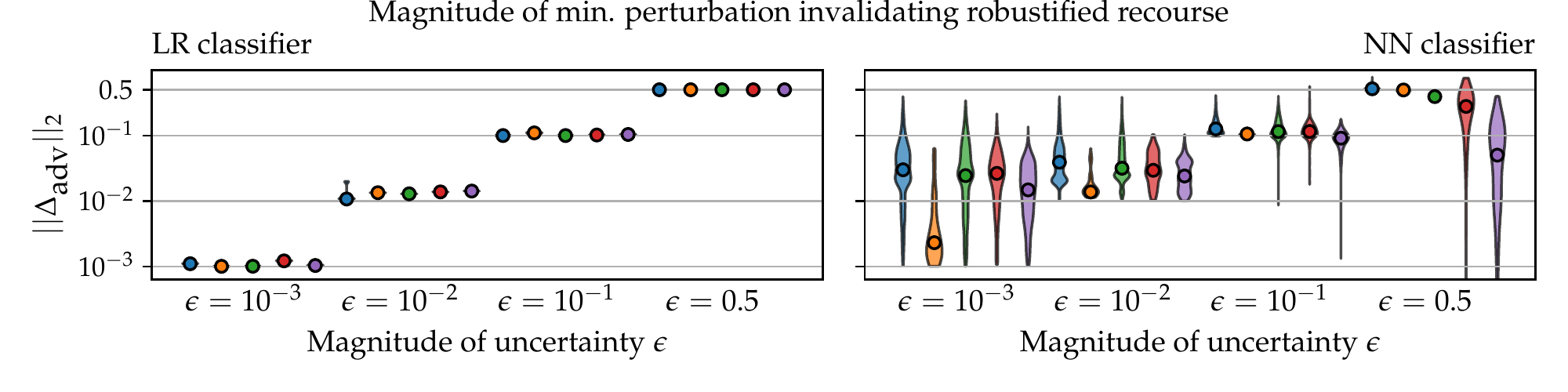}
    \end{adjustbox}}
    \end{center}
    \caption{Robustness of recourse actions that have been robustified against $\epsilon$ uncertainty. For LR classifiers, the proposed methods are very effective in generating robust recourse actions. For NN classifiers, some of the generated recourse actions may be fragile for reasonably large uncertainty $\epsilon\in\{0.1, 0.5\}$. Legend:
    \legendsquare{def1} COMPAS~~\legendsquare{def2}~Adult~~\legendsquare{def3}~Loan~~\legendsquare{def4}~German~~\legendsquare{def5}~Bail.}
        \label{fig:robpertb}
\end{figure*}

First, we empirically demonstrate that recourse methods which aim to generate minimum-cost recourse fail to be robust. To do so, we train the decision-making classifiers using expected risk minimization and generate recourse for the negatively classified individuals. We use the minimum-cost recourse methods of \citet{wachter2017counterfactual} and \citet{karimi2021algorithmic} for the causal and non-causal recourse setting, respectively. We then use the C\&W adversarial attack~\cite{carlini2017towards} to find the smallest additive intervention $\Delta_{\text{adv}}$ which invalidates the generated recourse action $a$, that is, such that $h(\mathbb{CF}(x_{\text{adv}}, a))=0$ for $x_{\text{adv}} = \mathbb{CF}(x, \Delta_{\text{adv}})$.

We present the experiment results in Figure~\ref{fig:pertb}. We observe that the recourse actions generated for both LR and NN classifiers are remarkably fragile, and can be invalidated under uncertainty ranging from $\epsilon = 10^{-2}$ to $10^{-6}$. For linear classifiers, recourse actions are particularly brittle, since the recourse problem is convex and thus its global minimum (which is provably fragile) can be accurately computed.

\subsection{Generating adversarially robust recourse}

We evaluate the effectiveness of the methods proposed in Section~5 in generating adversarially robust recourse against different uncertainty $\epsilon\in\{10^{-3}, 10^{-2}, 10^{-1}, 0.5\}$. We consider $\epsilon\in\{10^{-3}, 10^{-2}\}$ to be reasonably small uncertainty, $\epsilon=0.1$ significant uncertainty, and $\epsilon=0.5$ very large uncertainty. For each individual $x$ and level of uncertainty $\epsilon$, we first generate a recourse action $a$ which is robustified against $\epsilon$. We then use the C\&W attack to search for the smallest intervention $\Delta_{\text{adv}}$ to the features of the individual $x$ which invalidates the generated recourse action $a$. If the magnitude of such adversarial intervention is lower than the uncertainty $\epsilon$ against which the $a$ was robustified (i.e., $\norm{\Delta_{\text{adv}}} \leq \epsilon$), we can assert that the action $a$ is fragile.

We present the experimental results in Figure~\ref{fig:robpertb}. For linear models, all adversarial interventions $\Delta_{\text{adv}}$ found have magnitude larger than $\epsilon$. Thus, in the linear case our proposed method is effective in generating robust recourse. Furthermore, since all perturbations found are larger than~$\epsilon$ by only an arbitrarily small amount, the recourse actions generated are not only robust but also minimally costly (i.e., there is no over-robustification). For NNs models, our proposed method is effective in generating robust recourse actions against reasonably small uncertainty~$\epsilon$. For reasonably large uncertainty~$\epsilon$, however, some of the generated recourse actions may be fragile, since the non-covex inner maximization in Algorithm~1 is more likely to arrive to a local maximum rather than to its global maximum. Nonetheless, our proposed method generates substantially less brittle recourse actions compared to the standard minimum-cost recourse generation methods previously evaluated (Section~7.1).

\subsection{Actionable local linearity regularization}

We empirically evaluate whether training the decision-making classifier with the proposed ALLR regularizer facilitates the existence of adversarially robust recourse. We compare against the following approaches:
\begin{itemize}[topsep=0pt]
    \setlength\itemsep{0em}
    \item Empirical risk minimization (ERM): the standard choice for model training in the recourse literature. Equivalent to ALLR with zero regularization strength. 
    \item ERM using only actionable features (AF): facilitates the existence of recourse per Proposition~\ref{existence1}. Equivalent to ALLR with infinitely large regularization $\mu_2 \rightarrow \infty$.
    \item The approach of \citet{ross2021learning}, which makes actionable features more discriminative using the regularizer $\mathcal{R}(x) = \min_{\delta}\ell(h(x + m_{\mathcal{A}} \odot d), \mathbf{1})$.
\end{itemize}

First, we train five classifiers with different train and test splits for each of the considered approaches. For the negatively classified individuals, we generate standard recourse (i.e., $\epsilon=0$) and robust recourse against a significant amount of uncertainty $\epsilon=0.1$. We then evaluate the percentage of individuals for which robust recourse is found, as well as the mean cost of recourse. We additionally evaluate the predictive performance of the classifiers by computing their accuracy and Matthews correlation coefficient~(MCC).

We present the experimental results in Figure~\ref{fig:exp1}. For both LR and NN models, we find that our proposed regularizer is generally very effective in facilitating the existence of adversarially robust recourse, increasing the percentage of individuals for which robust recourse is found by up to~100\%. Furthermore, classifiers trained with ALLR generally offer recourse at a similar or lower cost compared to classifiers trained with ERM. The accuracy of the classifiers trained with ALLR may decrease by up to 3\% compared to classifiers trained with ERM, but does not decrease at all for three out of the five datasets considered. Using only actionable features (AF) for classification also significantly facilitates the existence of robust recourse, but often leads to poor predictive performance. Finally, while the regularization approach of \citet{ross2021learning} is very effective in providing low-cost recourse, ALLR is generally more effective in facilitating the existence of robust recourse.

 \begin{figure*}
      \begin{center}
    \scalebox{0.85}{\begin{adjustbox}{clip,trim=0.0cm 0.0cm 0.0cm 0.0cm}
            \includegraphics{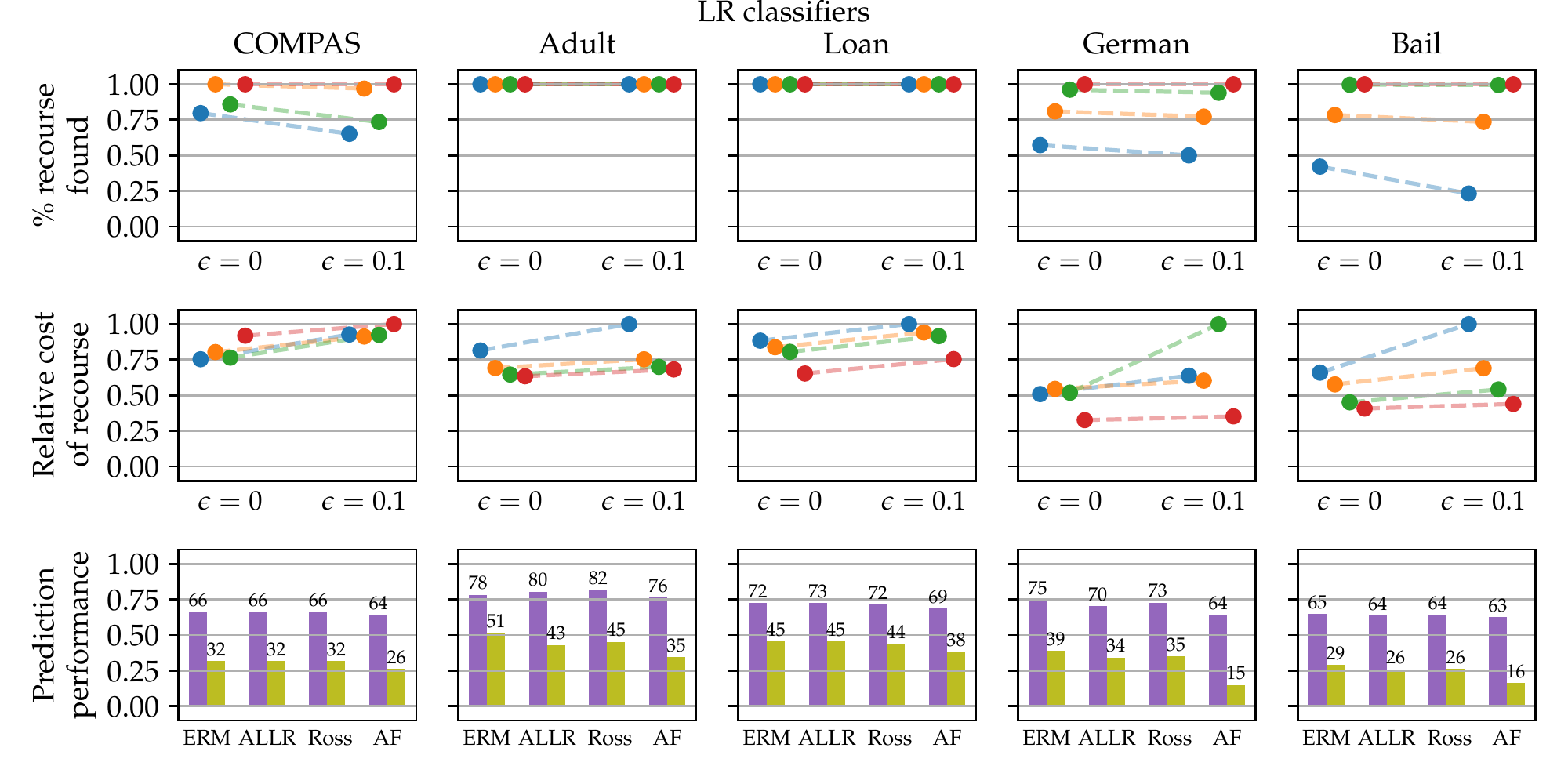}
    \end{adjustbox}}
    \phantom{space here}
    \scalebox{0.85}{\begin{adjustbox}{clip,trim=0.0cm 0.0cm 0.0cm 0.0cm}
            \includegraphics{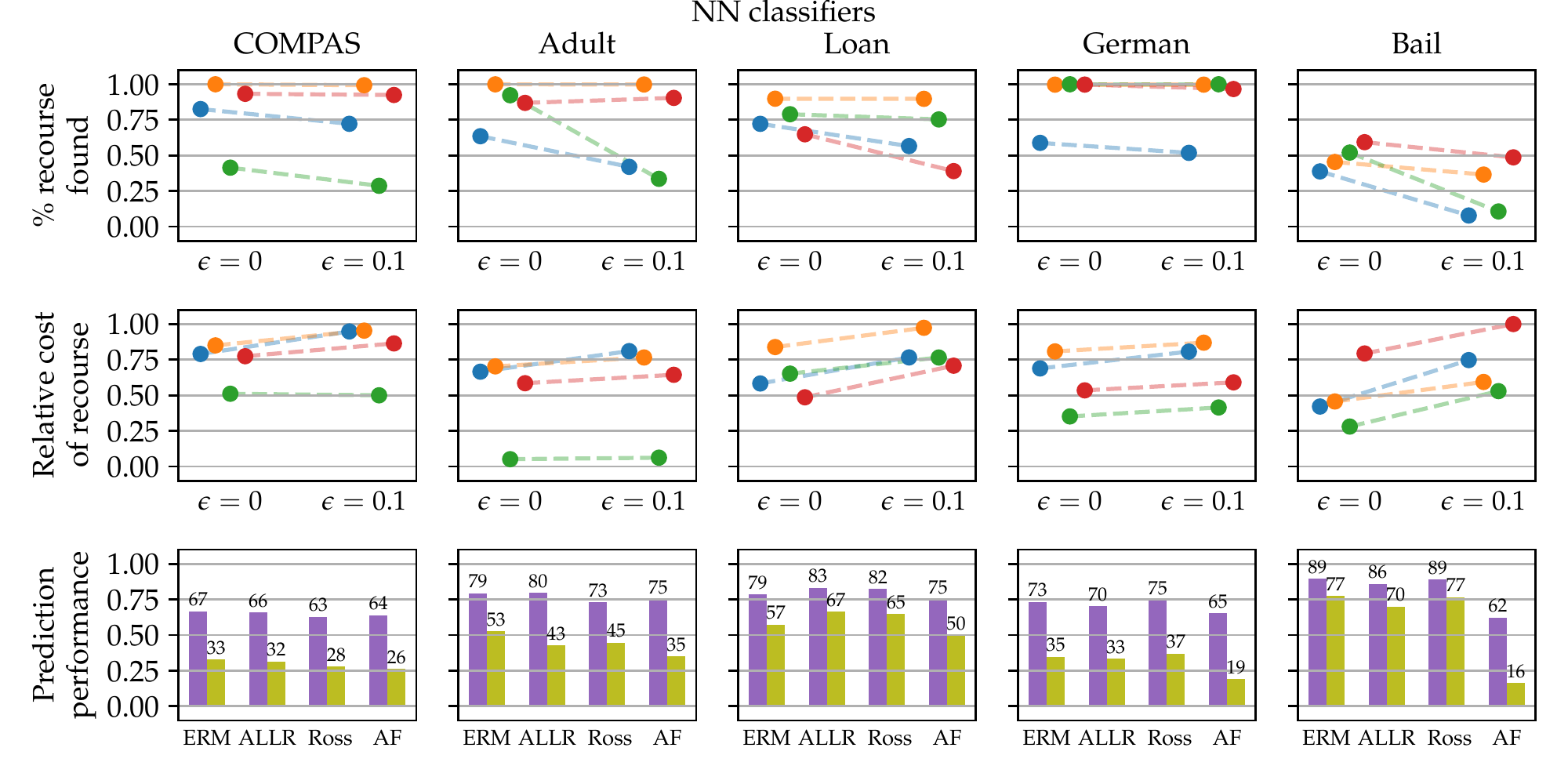}
    \end{adjustbox}}
    \end{center}
    \caption{For LR models, classifiers trained with ALLR provide robust recourse to a substantially larger percentage of individuals compared to classifiers trained using ERM. For NN models, ALLR is the most effective method in facilitating the existence of adversarially robust recourse. Legend: \legendsquare{def1} ERM~~\legendsquare{def2}~ALLR~~\legendsquare{def3}~Ross et al.~~\legendsquare{def4}~AF~~\legendsquare{def5}~Accuracy~~\legendsquare{def6}~MCC.}
        \label{fig:exp1}
\end{figure*}

\section{Conclusion}

Uncertainty in the recourse process is inevitable. Previously suggested \emph{ex-post} solutions to mitigate the effect of uncertainty in the recourse process may result in negative outcomes for both the decision-maker and the individual. We instead adopted an \emph{ex-anti} approach to the robustness of recourse by requiring recourse recommendations to be robust to uncertainty in the features of the individual seeking recourse. Alarmingly, we showed that minimum cost recourse is provably fragile to arbitrarily small uncertainty in the individual seeking recourse. To address this critical issue, we formulated the adversarially robust recourse problem and presented methods to generate adversarially robust recourse for linear and for differentiable classifiers. Finally, we derived sufficient conditions for the existence of adversarially robust recourse, and  we empirically demonstrated that regularizing the decision-making classifier to behave locally linearly and to rely more strongly on actionable features facilitates the existence of robust recourse.

\vfill

\paragraph{Acknowledgements} 
The authors thank Shalmali Joshi, Himabindu Lakkaraju, Martin Pawelczyk and Berk Ustun for helpful feedback and discussions. AHK acknowledges generous founding support from NSERC, CLS, and Google.

\newpage
\bibliography{biblio}
\bibliographystyle{icml2022}

\appendix
\onecolumn

\section{Uncertainties in the recourse process}\label{sec:uncertainties}

\tikzstyle{block} = [draw, rectangle, 
    minimum height=2em, minimum width=8em]
\tikzstyle{circ} = [draw, ellipse, dashed, radius=1em]
\tikzstyle{circ2} = [draw, circle, dashed, radius=1em]
\tikzstyle{circ3} = [draw, ellipse, radius=1em]
\tikzstyle{sum} = [draw, circle, node distance=1cm]
\tikzstyle{input} = [coordinate]
\tikzstyle{output} = [coordinate]
\tikzstyle{pinstyle} = [pin edge={to-,thin,black}]

\begin{figure*}[h]
\centering
\begin{tikzpicture}[auto, node distance=3.5cm,>=latex']
    \node [circ3, name=input, inner sep=2.2pt] {\small $\mathcal{M}$};
    
    \node [coordinate, below left=0.25cm and 0.5cm of input] (l1) {};
    \node [coordinate, right=13.5cm of l1] (r1) {};
    
    \node [block, right of=input, align=center, minimum width=7em, node distance=3.5cm] (train) {\small Train classifier};
    \node [circ3, right of=train, node distance=2.2cm, inner sep=2pt] (h) {\small $h$};
    
    \node [coordinate, below of=input, node distance=1cm] (m2) {};
    \node [coordinate, below left=2.4cm and 0.5cm of input] (l2) {};
    \node [coordinate, right=13.5cm of l2] (r2) {};
    
    \node [circ3, below of=train, node distance=1cm] (x) {\small $x$};
    
    \node [coordinate, right of=h, node distance=0.7cm] (hh) {};
    \node [coordinate, right of=m2, node distance=0.7cm] (mm) {};
    \node [coordinate, right of=x, node distance=0.7cm] (xx) {};
    
    \node [block, right of=x, minimum width=5em, node distance=4.2cm] (infer) {\small Inference};
    
    \node [block, below right=0.2cm and -1cm of infer, node distance=4cm, align=center, minimum width=6em] (rec) {\small Recourse \\ \small generation};
    
    \node [circ3, above left=-0.8cm and 1cm of rec, inner sep=1.5pt] (F) {\small $\mathcal{F}$};
    \node [circ3, above left=-0.95cm and 0.4cm of rec, inner sep=2pt] (c) {\small $c$};
    
    \node [coordinate, above right=0.0cm and 4.3cm of infer] (a1) {};
    \node [coordinate, below right=1cm and 4.3cm of infer] (a2) {};
    
    \node [block, below right = 0.4cm and -1.1cm of rec, minimum height=1.0cm, minimum width=1.5cm, align=center] (val) {\small Recourse \\ \small validation};
    
    \node [coordinate, right of=val, node distance=3.6cm] (val2) {};
    
    \node [circ, below of=m2, node distance=1.9cm, inner sep=1pt] (m3) {\small $\hat{\mathcal{M}}$};
    \node [circ2, below of=x, node distance=2.3cm, inner sep=2.5pt] (x2) {\small $\hat{x}$};
    \node [circ2, below of=h, node distance=3.7cm, inner sep=1pt] (h2) {\small $\hat{h}$};

    \draw [draw,->] (input) -- node {\small $p_{\text{train}}(x,y)$} (train);
    \draw [draw,->] (train) -- node {} (h);
    \draw [draw,-] (h) -- node {} (hh);

    \draw [draw,->] (hh) |- node {} ([yshift=5pt]infer);
    \draw [draw,->] (hh) |- node {} ([yshift=20pt]rec);
    \draw [draw,->] (xx) |- node {} ([yshift=24pt]rec);
    \draw [draw,->] (mm) |- node {} ([yshift=12pt]rec);
    
    \draw [draw,->] (a1) -- node [left] {\small Time} (a2);
    
    \draw [draw,->] (x) -- node {} (infer);
    \draw [draw,->] (infer) -| node {\small $\hat{y}=0$?} ([xshift=20pt]rec);
    \draw [draw,->] (rec) -| node {\small $a$} ([xshift=40pt]val);
    \draw [draw,->] (F) -- node {} ([yshift=-5pt]rec.west);
    \draw [draw,->] (c) -- node {} ([yshift=-10pt]rec.west);
    
    \draw [draw,->] (m2) -- node {$p_{\text{inference}}(x)$} (x);
    
    \draw [draw,-] (input) -- node {} (m2);
    \draw [dashed] (m2) -- node {} (m3);
    \draw [dashed] (x) -- node {} (x2);
    \draw [dashed] (h) -- node {} (h2);
    \draw [dashed] (m3) |- node {} (x2);
    \draw [dashed] (m3) |- node {} (h2);
    
    \draw [draw,->] (m3) -- node {} ([yshift=12pt] val.west);
    \draw [draw,->] (x2) -- node {} ([yshift=2pt] val.west);
    \draw [draw,->] (h2) -- node {} ([yshift=-9pt] val.west);
    
    \draw [draw,->] (val) --  (val2) node [above,xshift=25pt,pos=0.25] {\small $\hat{h}(\mathbb{CF}(\hat{x}, a, \hat{\mathcal{M}}))=1$?} node [below,xshift=10pt,pos=0.25] {};
   
    \draw [dashed, very thick] (l1) |- node {} (r1);
    \draw [dashed, very thick] (l2) |- node {} (r2);
    
\end{tikzpicture}
\caption{Overview of the recourse process. Uncertain elements are represented with dashed circles. Possible relations between uncertain elements are represented with non-bold dashed lines. Bold dashed lines represent temporal jumps.}
\label{fig:process}
\end{figure*}
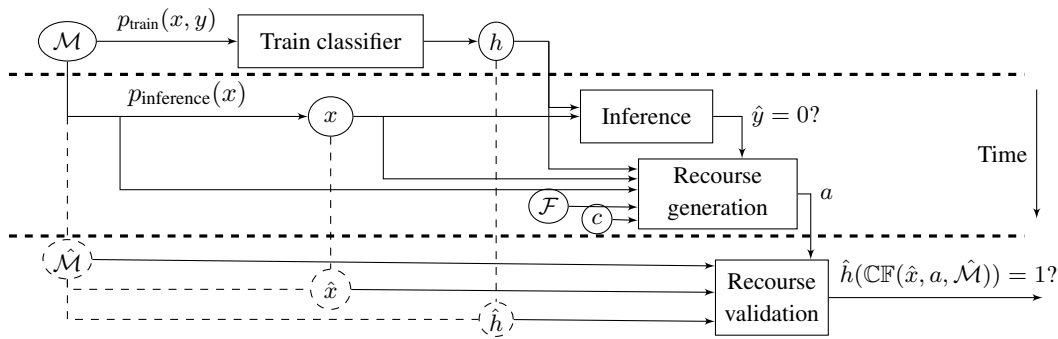

Uncertainties may arise throughout the recourse process, as depicted in Figure~\ref{fig:process}. Some well-studied sources of uncertainty in the classification setting naturally extend to algorithmic recourse. A great deal of the robust classification literature has focused on uncertainty in the inputs $x$ at inference time, which may arise due to the presence of noise~\cite{fawzi2016robustness,xu2009robustness}, adversarial manipulation~\cite{madry2017towards, szegedy2014intriguing} and other misrepresentations or errors in the data~\cite{zheng2016improving}. Regarding the classifier $h$, the optimization problem solved during model training often does not have a unique optimal solution and multiple models may perform equally well~\cite{breiman2001statistical, rudin2019stop}. Moreover, the temporal nature of recourse introduces a unique challenge: the circumstances under which recourse is generated may change by the time the individual is able to implement their prescribed recourse action. For instance, the distribution over inputs itself may change at inference time, under phenomena such as data-set shift~\cite{moreno2012unifying, quinonero2009dataset} or for tasks requiring out of distribution generalization~\cite{geirhos2020shortcut, muandet2013domain}. 

From a causal perspective, changes in the observational data distribution are a consequence of changes to the underlying SCM~\cite{buhlmann2020invariance}. Indeed, the data-generation process characterized by the SCM $\mathcal{M}$ may be imperfectly known~\cite{von2020fairness} or may dynamically change over time to some other SCM $\hat{\mathcal{M}}$. Consequently, the counterfactual individual resulting from the prescribed recourse intervention may also change. Furthermore, decision-makers may have to periodically retrain their models to prevent performance degradation due to the distribution shift resulting from a change in the SCM, producing further uncertainty over the future classifier $\hat{h}$~\cite{rawal2020algorithmic, upadhyay2021towards}. Finally, it may be unreasonable to expect the individual $x$ to not suffer changes outside its control over an extended period of time~\cite{venkatasubramanian2020philosophical}, leading to uncertainty in the future individual $\hat{x}$. Thus, acting on the prescribed recourse may not lead to favorable classification due to changes to the SCM $\hat{\mathcal{M}}$, classifier $\hat{h}$, and/or factual individual $\hat{x}$.

\section{Counterfactual neighborhoods adapt to the local geometry of the data manifold}\label{sec:neig}

\begin{figure}[h]
\centering
      \captionsetup{width=.95\linewidth}
     \scalebox{0.85}{\begin{adjustbox}{clip,trim=1cm 0.6cm 0.2cm 1.cm}
            \includegraphics{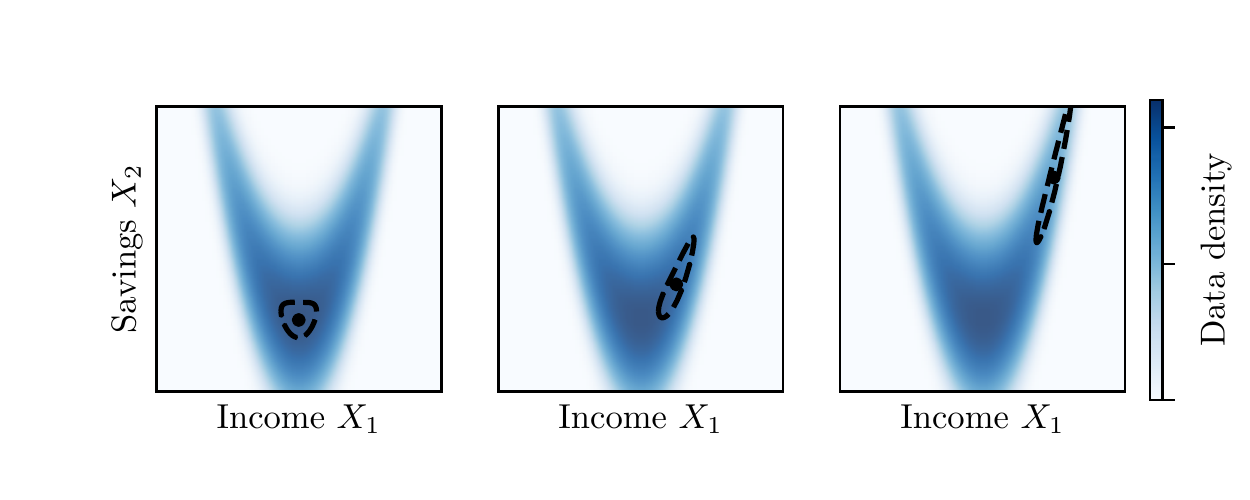}
    \end{adjustbox}}
 \caption{Counterfactual neighborhoods of similar individuals for the SCM $X_1=U_1$, $X_2 = X_1^2 + U_2$.}
 \label{fig:simnon}
\end{figure}%

The SCM $\mathcal{M}$ amounts to a generative model of the data and thus (approximately) captures the underlying data manifold. Therefore, the individuals in the uncertainty set $B(x)$ are realistic in the sense that they lie within the data manifold (i.e., have sufficient data support). For non-linear SCMs, the shape of $B(x)$ adapts to the local geometry of the manifold, as illustrated in Figure~\ref{fig:simnon}. Further studying this behavior is an interesting research direction for future work.
 
\begin{figure}[h]
\centering
\begin{minipage}{.47\textwidth}
  \centering
  \scalebox{.6}{\begin{adjustbox}{clip,trim=0.1cm 0.3cm 0.2cm 0.2cm}
            \includegraphics{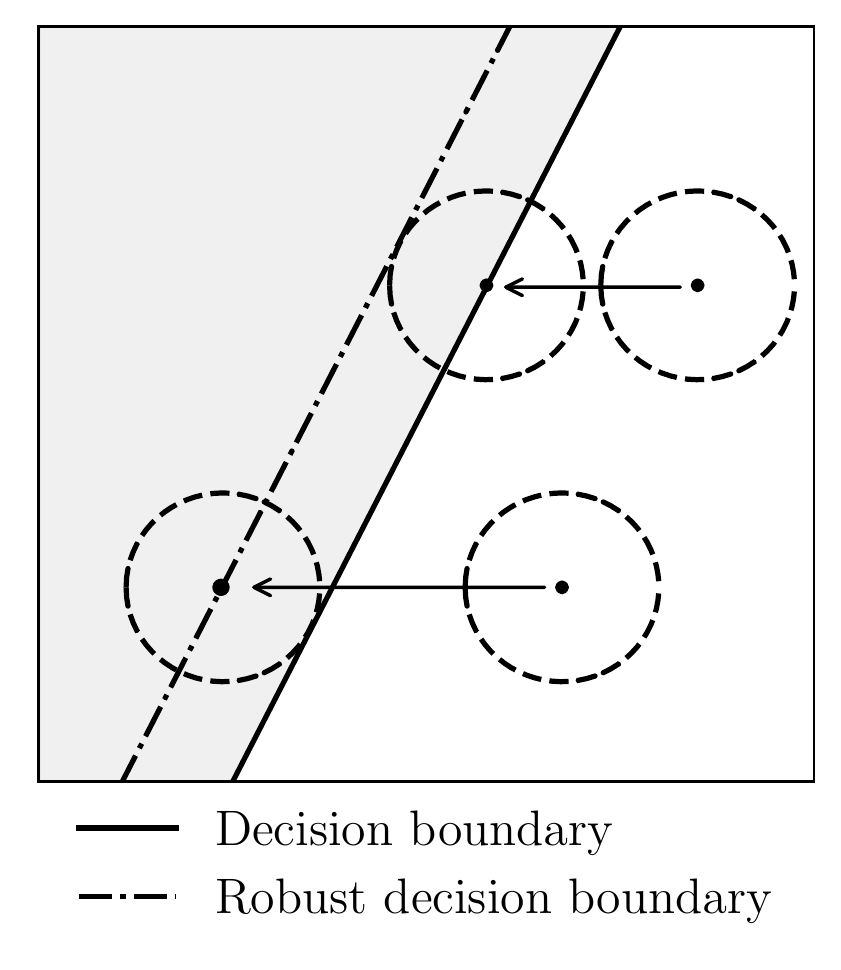}
    \end{adjustbox}}
  \captionof{figure}{Standard recourse with respect to the robust decision boundary results in adversarially robust recourse.}
  \label{fig:lin-robust-bdry}
\end{minipage}%
\hspace{.04\textwidth}
\begin{minipage}{.47\textwidth}
  \centering
  \scalebox{.6}{\begin{adjustbox}{clip,trim=0.1cm 0.3cm 0.2cm 0.2cm}
            \includegraphics{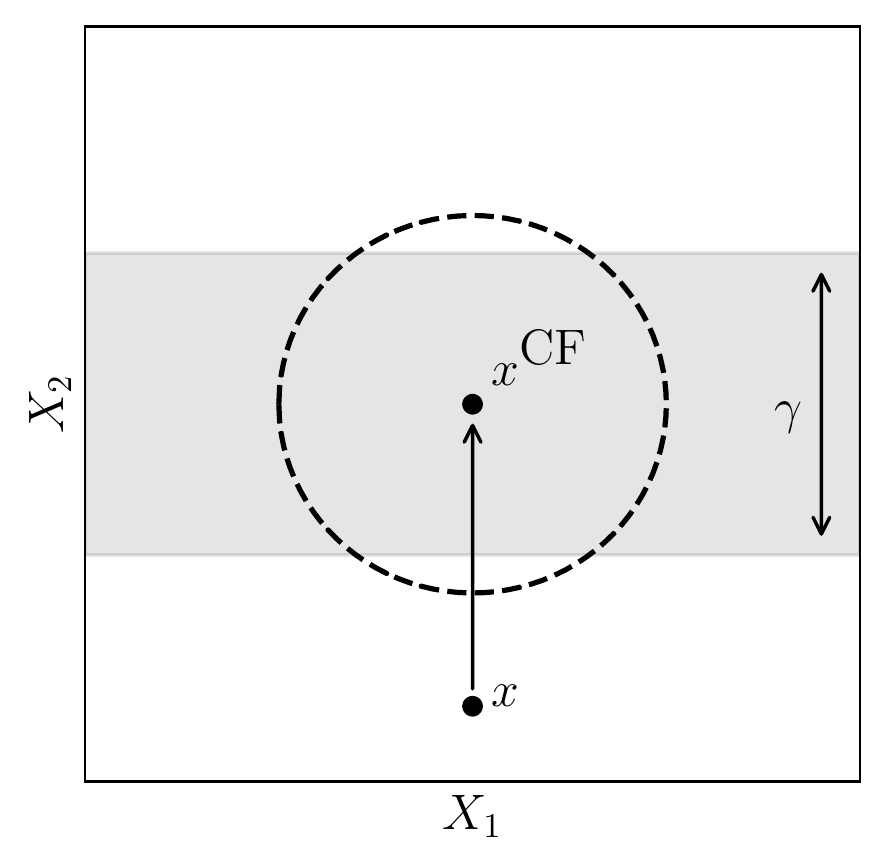}
    \end{adjustbox}}
  \captionsetup{width=.95\linewidth}
  \captionof{figure}{Classifier of Example~\ref{ex1}. The shaded area is the favourably classified region of the feature space}
  \label{fig:example1} 
\end{minipage}
\end{figure}

\section{Proofs}

\subsection{Theorem~\ref{th:fragile}}

Let $a^*(x)=do(\mathbf{X}_{\mathcal{I}}=x_{\mathcal{I}}+\theta^*)$ be a minimum-cost recourse action for some classifier $h$ and individual $x\in\mathcal{X}$. Assume that $a^*$ is a robust recourse action for the uncertainty set $B(x) = \{\mathbb{CF}(x, \Delta) \;|\; \norm{\Delta} \leq \epsilon\}$. Let us choose some intervened-upon feature $\mathbf{X}_{\mathcal{I}_j}$ that is not a causal ancestor of any other intervened-upon feature $\mathbf{X}_{\mathcal{I}_i} \; \forall i \neq k$. For a DAG causal graph, at least one such feature $\mathbf{X}_{\mathcal{I}_j}$ must clearly exist. We will consider perturbations along precisely this feature $\mathbf{X}_{\mathcal{I}_j}$.

Let $e^j\in\mathbb{R}^{|\mathcal{I}|}$ be a standard basis vector such that $e_j^j=1$ and $e_i^j = 0 \; \forall i \neq j$. Consider the perturbation $\delta=\alpha e^j\sign(\theta_j)$ for any $\alpha \leq \epsilon$. The modified action $a'(x)=do(X_{\mathcal{I}}=x_{\mathcal{I}}+\theta^*-\delta)$ is a valid recourse action, since
\begin{equation}
    \begin{split}
        h(\mathbb{CF}\left(x, a'\right)) &= h(\mathbb{CF}\left(x, do(X_{\mathcal{I}}=x_{\mathcal{I}}+\theta^*-\delta)\right)\\
        &= h\left(\mathbb{CF}\left(\mathbb{CF}(x, -\delta), do(X_{\mathcal{I}}=x_{\mathcal{I}}+\theta^*)\right)\right) \\
        &= h\left(\mathbb{CF}\left(\mathbb{CF}(x, -\delta), a^*\right)\right)\\
        &= 1
    \end{split}
\end{equation}
where the second equality holds given that the SCM $\mathcal{M}$ is an additive noise model and $\mathbf{X}_j$ is not a causal ancestor of the intervened-upon features $\mathbf{X}_{\mathcal{I}}$. The last equality holds per the assumption that $a^*$ is an adversarially robust recourse action, and since $\norm{\delta} \leq \epsilon$. If the cost function is subadditive, then it must be that 
\begin{equation}
    \begin{split}
        c(x, a') &= c(x, \theta^* + \theta - \delta) \\
        &< c(x, \theta^* + \theta) \\
        & = c(x, a^*)
    \end{split}
\end{equation}
Thus, $a'$ is a valid recourse action which has strictly lower cost than $a^*$. This is a contradiction on $a^*$ being a minimum-cost recourse action, which stems from the assumption that $a^*$ is adversarially robust. Consequently, the minimum-recourse action $a^*$ is provably fragile to arbitrarily small (i.e., $\epsilon > 0$) perturbations to the features of the individual $x$ seeking recourse. 

\subsection{Proposition~\ref{noexistence}}

\begin{example}
Consider $\mathcal{X} = \mathbb{R}^2$ and $h(x) = \sin(2\gamma\pi^{-1} x_2) \; \geq \; 0$ for $0 < \gamma < \epsilon $. Assume that all features are independently modifiable, and consider the uncertainty set $B(x) = \{x + \Delta \;| \; \norm{\Delta}_2 \leq \epsilon\}$. Whilst there exists some recourse recommendation for all $x \in \mathcal{X}$, there does not exist any adversarially robust recourse action for any $x \in \mathcal{X}$. See Figure~\ref{fig:example1}.
\label{ex1}
\end{example}

\subsection{Proposition~\ref{existence1}}

Let $h$ be a classifier for which there exists some robustly classified individual $x^+\in\mathcal{X}$ such that $h(x^+)=1 \; \forall x' \in B(x^+)$, where $B(x^+)=\{\mathbb{CF}(x^+, \Delta)|\norm{\Delta}\leq \epsilon\}$. For any given individual $x$, the action $a(x)=do\left(\mathbf{X}=x+(x^+-x)\right)$ results in the counterfactual individual $\mathbb{CF}(x, a)=x^+$. The action $a$ is a recourse action, since $h(x^{\textrm{CF}})=h(x^+)=1$. Moreover, the action $a$ is feasible, since all features are actionable by assumption. For any $x'\in B(x)$ it holds that
\begin{equation}
    \{\mathbb{CF}(x', a) \; | \; x'\in B(x) \} = \{x' + x^+ - x \; | \; x'\in B(x)\} = B(x^+)
\end{equation}
where the last equality holds since the SCM $\mathcal{M}$ is linear. Consequently, it holds that
\begin{equation}
        h\left(\mathbb{CF}\left(x', a\right)\right)=1 \; \forall x'\in B(x) \iff h(x')=1 \; \forall x'\in B(x^+)
\end{equation}
since the equality in the right hand side holds by assumption that $x^+$ is robustly classified, it follows that the action $a$ is an adversarially robust recourse action.

\subsection{Proposition~\ref{existence2}}\label{sec:proofex2}

By assumption the classifier $h(x)=\inner{w,x} \geq b$ and SCM $\mathcal{M}$ are linear, and thus per Proposition~\ref{linear} a recourse action $a$ is adversarially robust if it holds that $\inner{w, \mathbb{CF}(x, a)} \geq b'$ for some $b'>b$. By assumption, there exists some feature $\mathbf{X}_j$ such that $\mathbf{X}_j$ is actionable and unbounded. Consider the recourse action $a(x)=do(\mathbf{X}_j=x_j+\theta)$ for $\theta\in\mathbb{R}$. Due to the linearity assumptions on the SCM, $\mathbb{CF}(x, a) = x + \theta v$ for some $v\in\mathbb{R}^n$. Then, $\inner{w, \mathbb{CF}(x, a)} = \inner{w, x + \theta v} = \inner{w,x}+\theta\inner{w,v}$. A robust recourse action is equivalent to any $\theta$ such that $ \theta \inner{w,v}\geq b'-\inner{w,x}$. If $\inner{w,v}\neq0$, then clearly it is possible to set $\theta$ to have arbitrarily large magnitude and same sign as $\inner{w,v}$, such that the inequality above is met. Since $\mathbf{X}_j$ is actionable and unbounded, $a(x)=do(\mathbf{X}_j=x_j+\theta)$ is a feasible action. Consequently, $a$ is a robust recourse action.

We argue that the requirement $\inner{w,v}\neq0$ is a mild condition on the weights of the classifier, precisely the non-trivial case where the weights of the classifier are not chosen adversarially to the SCM. Such condition is the causal counterpart of the trivial requirement that $w_j \neq 0$ in the statements presented by \citet{ustun2019actionable} for the non-causal recourse setting.

\subsection{Proposition~\ref{linear}}

The adversarially robust recourse problem is defined as
\begin{equation}
\min_{a(x)=do(X_{\mathcal{I}}=x_{\mathcal{I}}+\theta)} \; c(x, a) \quad \text{s.t.} \quad a\in\mathcal{F}(x)\; \wedge\; h\left(\mathbb{CF}\left(x', a\right)\right)=1 \quad \forall \; x' \in B(x)
\end{equation}

For a linear classifier $h(x) = \inner{w, x} \geq b$ the uncertain constraint is equivalent to
\begin{equation}
    h\left(\mathbb{CF}\left(x', a\right)\right)=1 \quad \forall \; x' \in B(x) \iff \left(\min_{x'\in B(x)} \inner{w,\mathbb{CF}\left(x', a\right)}\right) \geq b
     \label{eq:p13}
\end{equation}

Under the assumption that the SCM $\mathcal{M}$ is linear, for any plausible individual $x'=\mathbb{CF}(x, \Delta)$ it holds that

\begin{equation}
\begin{split}
    \mathbb{CF}(x', a) &= \mathbb{S}^{a}\left(\mathbb{S}^{-1}\left(x'\right)\right)\\
    &=\mathbb{S}^{a}\left(\mathbb{S}^{-1}\left(\mathbb{CF}(x, \Delta)\right)\right)\\
    &=\mathbb{S}^{a}\left(\mathbb{S}^{-1}\left(\mathbb{S}^{\Delta}\left(\mathbb{S}^{-1}(x)\right)\right)\right) \\ &= \mathbb{S}^{a}\left(\mathbb{S}^{-1}\left(\mathbb{S}\left(\mathbb{S}^{-1}(x)+\Delta\right)\right)\right) \\ &= \mathbb{S}^{a}\left(\mathbb{S}^{-1}(x)+\Delta\right)\\
    &= \mathbb{S}^{a}\left(\mathbb{S}^{-1}(x)\right) +\mathbb{S}^{a}\left(\Delta\right)\\
    &= \mathbb{CF}(x, a) +J_{\mathbb{S}^{\mathcal{I}}}\Delta\\
    \label{eq:p14}
\end{split}
\end{equation}
where $J_{\mathbb{S}^{\mathcal{I}}}$ denotes the Jacobian of the interventional mapping $\mathbb{S}^{\mathcal{I}}$ resulting from hard-intervening on the variables $\mathbf{X}_\mathcal{I}$.

Then  for the uncertainty set $B(x) = \{\mathbb{CF}(x,a) \; | \; \norm{\Delta} \leq \epsilon\}$ the uncertain constraint in Equation~\ref{eq:p13} is equivalent to
\begin{equation}
    \begin{split}
        \min_{x'\in B(x)}  \inner{w,\mathbb{CF}\left(x', a\right)} &= \min_{\norm{\Delta}\leq \epsilon}  \inner{w,\mathbb{CF}\left(x, a\right) + J_{\mathbb{S}^\mathcal{I}}\Delta}\\
        &= \inner{w,\mathbb{CF}\left(x, a\right)} + \min_{\norm{\Delta}\leq \epsilon} \inner{w,J_{\mathbb{S}^\mathcal{I}}\Delta}\\
        &= \inner{w,\mathbb{CF}\left(x, a\right)} - \norm{J^T_{\mathbb{S}^\mathcal{I}}w}^*\epsilon
    \end{split}
\end{equation}

Consequently, the adversarially robust recourse problem reduces to
\begin{equation}
    \min_{a(x)=do(X_{\mathcal{I}}=x_{\mathcal{I}}+\theta)} \; c(x, a) \quad \text{s.t.} \quad a\in\mathcal{F}(x)\; \wedge\;   \inner{w,\mathbb{CF}\left(x, a\right)}\geq b + \norm{J^T_{\mathbb{S}^\mathcal{I}}w}^*\epsilon
    \label{eq:p103}
\end{equation}
This is equivalent to the standard recourse problem for the classifier $h'(x) = \inner{w, x} \geq b + \norm{J^T_{\mathbb{S}^\mathcal{I}}w}^*\epsilon$.  See Figure~\ref{fig:lin-robust-bdry}. 

\subsection{Proposition~\ref{differentiable}}

For a classifier $h(x) = \tilde h(x) \geq b$ where $\tilde h: \mathcal{X} \rightarrow \left(0, 1\right]$ the uncertain constraint is equivalent to
\begin{equation}
\begin{split}
    h\left(\mathbb{CF}\left(x', a\right)\right)=1 \quad \forall \; x' \in B(x) &\iff \left(\min_{x'\in B(x)} \tilde h\left(\mathbb{CF}\left(x', a\right)\right)\right) \geq b \\
    &\iff \left(\min_{x'\in B(x)} \log \tilde h\left(\mathbb{CF}\left(x', a\right)\right)\right) \geq \log b \\
    &\iff \left(-\max_{x'\in B(x)} -\log \tilde h\left(\mathbb{CF}\left(x', a\right)\right)\right) \geq \log b \\
    &\iff \log b + \max_{x'\in B(x)} \ell \left( \tilde h\left(\mathbb{CF}\left(x', a\right)\right), \mathbf{1}\right) \leq 0 \\
\end{split}
\end{equation}
The corresponding Lagrangian is then
\begin{equation}
    L(a, \lambda) = c(x, a) + \lambda \left(\log b + \max_{x'\in B(x)} \ell \left( \tilde h\left(\mathbb{CF}\left(x', a\right)\right), \mathbf{1}\right)\right)
\end{equation}

Therefore, adversarially robust recourse problem is equivalent~\cite{boyd2004convex} to
\begin{equation}
    \min_{a(x)=do(X_{\mathcal{I}}=x_{\mathcal{I}}+\theta)\in\mathcal{F}(x)}\max_{\lambda \geq 0} \; L(a, \lambda) 
\end{equation}

\subsection{Proposition~\ref{th:cost} and extension to the causal setting}\label{bound}

Let the classifier $h(x) = \inner{w, x} \geq b$ and SCM $\mathcal{M}$ be linear. By assumption, the action $a(x)=do\left(\mathbf{X}_{\mathcal{I}} = x_{\mathcal{I}}+\theta\right)$ is a recourse action for $h$, and the features $\mathbf{X}_\mathcal{I}$ are actionable and unbounded. Consider a modified action of the form $a'(x)=do\left(\mathbf{X}_{\mathcal{I}} = x_{\mathcal{I}}+(1+\beta)\theta\right)$ where the individual $x$ is asked to intervene on the features $x_{\mathcal{I}}$ by an additional factor~$\beta$. Per the assumption that $x_{\mathcal{I}}$ are unbounded, the modified action $a'$ is actionable. Per Proposition~\ref{linear}, a sufficient condition for $a'$ to be an adversarially robust recourse action against the uncertainty set $B(x) = \{\mathbb{CF}(x, \Delta) \; | \; \norm{\Delta} \leq \epsilon\}$ is
\begin{equation}
    \inner{w,\mathbb{CF}\left(x, a'\right)}\geq b + \norm{J^T_{\mathbb{S}^\mathcal{I}}w}^*\epsilon
    \label{eq:p20}
\end{equation}
 Since the SCM $\mathcal{M}$ is assumed linear, it holds that $\mathbb{CF}(x, a')=x + J_{\mathbb{S}^I}(1+\beta\epsilon)\theta$ (first order Taylor expansion). Then,
\begin{equation}
    \begin{split}
         \inner{w,\mathbb{CF}\left(x, a'\right)} &= \inner{w, x + (1+\beta\epsilon)J_{\mathbb{S}^{\mathcal{I}}}\theta)}\\
         &=\inner{w, x + J_{\mathbb{S}^I}\theta}+\beta\epsilon\inner{w, J_{\mathbb{S}^{\mathcal{I}}}\theta}\\
         &=\inner{w, \mathbb{CF}(x, a)}+\beta\epsilon\inner{w, J_{\mathbb{S}^{\mathcal{I}}}\theta}\\
         &\geq b + \beta\epsilon\inner{w, J_{\mathbb{S}^{\mathcal{I}}}\theta}
    \end{split}
    \label{eq:p21}
\end{equation}
where the last inequality follows by assumption that $a$ is a recourse action for $h$. Consequently, we seek a condition on $\beta$ such that Equation~\ref{eq:p21} satisfies the inequality in Equation~\ref{eq:p20}. In particular, it holds that
\begin{equation}
    \beta = \dfrac{ \norm{J^T_{\mathbb{S}^\mathcal{I}}w}^*}{\inner{w, J_{\mathbb{S}^{\mathcal{I}}}\theta}}\epsilon \implies \inner{w,\mathbb{CF}\left(x, a'\right)}\geq b + \norm{J^T_{\mathbb{S}^\mathcal{I}}w}^*\epsilon \implies \; a' \; \text{is adversarially robust recourse.}
\end{equation}

Since $a$ is a recourse action, it must hold that $\inner{w,J_{\mathbb{S}^\mathcal{I}}} > 0$. Consequently, $0 < \beta < \infty$, and thus the additional change $\beta$ to $\mathbf{X}_{\mathcal{I}}$ required by the individual to robustify the recourse action $a$ is finite. 

Under the assumption that the cost function is subadditive, it follows that
$c(x, a')  \leq (1 + \beta) \; c(x, a)$. Rearranging, we obtain the following upper bound on the additional effort required by the individual to robustify the recourse action $a$:

\begin{equation}
    \dfrac{c(x, a') - c(x, a)}{c(x, a)} \leq \beta, \quad \text{for} \quad \beta =
    \dfrac{ \norm{J^T_{\mathbb{S}^\mathcal{I}}w}^*}{\inner{w, J_{\mathbb{S}^{\mathcal{I}}}\theta}}\epsilon
\end{equation}

Under the IMF assumption (i.e., in the non-causal setting), $J_{\mathbb{S}^\mathcal{I}}$ is the identity matrix $I$. Let $m_{\mathcal{A}} \in [0,1]^n$ (resp. $m_{\tilde{\mathcal{A}}}$) be the mask vector for the set of actionable features $\mathcal{A}$ (resp. unactionable features~$\tilde{\mathcal{A}}$). Trivially, $m_{\mathcal{A}} + m_{\tilde{\mathcal{A}}} = \mathbf{1}$. Then
\begin{equation}
    \begin{split}
        \dfrac{c(x, a') - c(x, a)}{c(x, a)} &\leq \dfrac{ \norm{w}^*}{\inner{w, \theta}}\epsilon \\
        &= \dfrac{\norm{(m_{\mathcal{A}} + m_{\tilde{\mathcal{A}}}) \odot w}^*}{\inner{(m_{\mathcal{A}} + m_{\tilde{\mathcal{A}}}) \odot w, \theta}}\epsilon\\
        &\leq \dfrac{\norm{m_{\mathcal{A}}\odot w}^* + \norm{m_{\tilde{\mathcal{A}}} \odot w}^*}{\inner{m_{\mathcal{A}} \odot w, \theta} + \inner{m_{\tilde{\mathcal{A}}} \odot w, \theta}}\epsilon\\
        &= \dfrac{\norm{m_{\mathcal{A}}\odot w}^* + \norm{m_{\tilde{\mathcal{A}}} \odot w}^*}{\inner{m_{\mathcal{A}} \odot w, \theta}}\epsilon\\
    \end{split}
\end{equation}

where the last equality follows from the fact that the prescribed feature change $\theta$ necessarily only affects actionable features, and thus $m_{\tilde{\mathcal{A}}} \odot \theta = \mathbf{0}$. We thus arrive at the upper bound presented in Proposition~\ref{th:cost}. 

\section{Relation between adversarial robustness and robustness of the cost of recourse}\label{sec:robcost}

In this work we focus on generating adversarially robust recourse. That is, we aim to offer a single recourse action that remains valid for plausible individuals similar to the individual $x$ seeking recourse. A related notion of robustness is the robustness of the cost of recourse: the extent to which similar individuals are offered different but similarly costly recourse recommendations~\cite{von2020fairness, slack2021counterfactual, artelt2021evaluating}, that is:
\begin{equation}
    \delta(x) = \sup_{x\in B(x)} \left|\left(\min_{a\in \mathcal{F}(x')} c(x', a) \text{ s.t. } h(\mathbb{CF}(x', a)) = 1\right) - \left(\min_{a\in \mathcal{F}(x)} c(x, a) \text{ s.t. } h(\mathbb{CF}(x, a)) = 1 \right)\right|
\end{equation}

First, we show that the robustness of the cost of recourse is necessary for the existence of adversarially robust recourse

\begin{proposition} Let $x \in \mathcal{X}$ be some individual for which there exists at least one recourse action, and let $B(x)$ be the uncertainty set. If $\delta(x) = \infty$, then there does not exist any adversarially robust recourse action for the individual~$x$.
\end{proposition}
\begin{proof} By assumption there exists some recourse action for $x$, and therefore $\delta(x)=\infty$ implies that there exists some individual $x'\in B(x)$ for which there does not exist any recourse action. It follows that there does not exist any robust recourse action $a^*$ for the individual $x$, since $a^*$ necessarily cannot be a recourse action for $x'$. 
\end{proof}

Additionally, under certain conditions, the cost of adversarially robust recourse for an individual $x$ can be used to upper bound the robustness of the cost of recourse $\delta(x)$ around that individual $x$.

\begin{proposition} For some classifier $h$, individual $x\in\mathcal{X}$, and uncertainty set $B(x)=\{\mathbb{CF}(x, \Delta) \; | \; \norm{\Delta} \leq \epsilon\}$, if
\begin{enumerate} [label=(\roman*),topsep=0pt,itemsep=-1ex,partopsep=1ex,parsep=1ex]
\setlength\itemsep{0em}
\item Similar individuals have equal actionability constraints $\mathcal{F}(x) = \mathcal{F}(x') \; \forall x' \in B(x)$.
\item The cost of any given recourse action $a$ is the same for all similar individuals, that is, $c(x, a) = c(x', a) \; \forall x' \in B(x)$.
\end{enumerate}
Then, if there exists some adversarially robust recourse action $a^*$, the robustness of the cost of recourse is upper bounded by the cost of the adversarially robust recourse action $a^*$, that is, $\delta(x) \leq c(x, a^*)$. 
\end{proposition}
\begin{proof} For notational simplicity, let $r(x) = \min_{a\in \mathcal{F}(x)} c(x, a) \text{ s.t. } h(\mathbb{CF}(x, a)) = 1$. Then
\begin{equation}
    \begin{split}
        \delta (x) &= \max\left(\max_{x'\in B(x)} r(x') - r(x), r(x) - \min_{x'\in B(x)} r(x')\right)\\
        &\leq \max\left(c(x, a^*) - r(x), r(x)\right) \leq \max\left(c(x, a^*), r(x)\right) \leq c(x, a^*)
    \end{split}
\end{equation}
\end{proof}

\section{Implementation details for the experiments}\label{sec:datasets}

\subsection{Datasets}

\subsubsection{Features used}\label{sec:feats}

COMPAS~\cite{larson2016we}: we use the features ``age'', ``race'', ``sex'' and ``priors count'', which are the variables in the causal graph of \citet{nabi2018fair}. We consider ``priors count'' actionable. As actionability constraints, we assume that ``priors count'' is non-negative and can only be decreased.

Adult~\cite{kohavi1996adult}: we use the features ``sex'', ``age'', ``native-country'', ``marital-status'', ``education-num'', ``hours-per-week'', which are the variables in the causal graph of \citet{nabi2018fair}. We consider ``education-num'' and ``hours-per-week'' actionable. We assume that ``education-num'' is bounded within $\left[1, 16\right]$ (i.e., 16 discrete categories, but for simplicity we assume the variable to be real-valued), and that ``hours-per-week'' is bounded within $\left[0, 100\right]$.

South German Credit: we use the version of the dataset introduced by \citet{groemping2019south}, which corrects some coding errors in the original of German Credit dataset from the UCI repository~\cite{murphy1994uci}. We use all 20 features, with target variable ``credit risk''. We assume ``duration'' and ``amount'' as actionable, since these are the only real-valued features (together with ``age'', which we assume unactionable). As actionability constraints, we require ``duration'' and ``amount'' to be positive.

Bail~\cite{schmidt1988predicting}: we use all features except ``file'' (the original test-train split) and ``time'' (which directly gives away the recidivism). We use as target variable ``recid'' (recidivism). For consistency with \citet{ross2021learning}, we consider the features ``school'' (number of years of formal school completed) and ``rule'' (number of prison rule violations reported) as actionable. For the actionability constraints, we assume that ``school'' is bounded within $\left[1, 19\right]$ (we assume the variable to be real-valued) and can only be increased, and that ``rule'' is non-negative and can only be decreased.

Loan: we use the semi-synthetic SCM described in \citet{karimi2020algorithmic} Appendix E.1.2. We sample 1000 individuals. We assume ``education'', ``income'' and ``savings'' to be actionable. For the actionability constraints, we assume that all three actionable variables can only be increased. Additionally, we assume that ``education'' is bounded by above to the maximum education level observed in the training data.

\subsubsection{Standard deviation of the features used}\label{sec:stds}

We standardize real-valued features, and define the uncertainty set $B(x) = \{\mathbb{CF}(x, \Delta) \;|\; \norm{\Delta}_2 \leq \epsilon \}$ with respect to the standardized features. Therefore, any given feature can be perturbed by at least $\pm~\epsilon~\times $ the feature's standard deviation. To provide the reader with a better sense of the magnitude of uncertainty for each $\epsilon$ considered, we list the standard deviation of the real-valued features of each dataset. 

\emph{COMPAS}$\;\;$ ``Age'': 11.7 years, ``Number of convictions'': 4.74 $\quad$\emph{Adult}$\;\;$ ``Age'' 13.6 years, ``Education number'' 2.57, ``Hours per week'' 12.3 $\quad$\emph{German}$\;\;$ ``Age'': 11.3 years, ``Loan duration'': 12.0 months, ``Loan amount'': 2822 Deutsche Mark. $\quad$\emph{Bail}$\;\;$ ``Age'': 9.66 years, ``Number of priors'': 2.91 ``Years of school'': 2.46, ``Rule violations'': 2.41, ``Time served': 22.1 months, ``Follow-up period'': 3.44 months. $\quad$ \emph{Loan}$\;\;$ ``Age'': 11.0 years. The rest of the features lack meaningful units.


For instance, for the COMPAS dataset, robustifying against $\epsilon=0.1$ uncertainty is equivalent to robustifying against perturbations to the feature ``age'' of at least $\pm~0.1 \times 11.7 = 1.17$ years.

\subsection{Training of the classifiers}

We train the decision-making classifiers $h(x)$ using one of the following training objectives:
\begin{itemize}[topsep=0pt]
\setlength\itemsep{0em}
    \item Empirical risk minimization (ERM): $\min_{\psi} \mathbb{E}\left[\ell(h_{\psi}(x), y)\right]$
    \item ERM with actionable features (AF): $\min_{\psi} \mathbb{E}\left[\ell(h_{\psi}(m_{\mathcal{A}}\odot x), y)\right]$
    \item ALLR: $\min_{\psi} \mathbb{E}\left[\ell(h_{\psi}(x), y)+ \mu_2 \norm{m_{\tilde{\mathcal{A}}}\odot\nabla_x h(x)} + \mu_1  \max_{\norm{\delta} \leq \epsilon}|h(x+\delta) - \inner{\delta, \nabla_x  h(x)} -  h(x)|\right]$
    \item \citet{ross2021learning}: $\min_{\psi} \mathbb{E}\left[\ell(h_{\psi}(x), y)+ \mu \min_{\delta}\ell(h(x + m_{\mathcal{A}} \odot d), \mathbf{1})\right]$
\end{itemize}

We use the binary cross-entropy loss as the loss function $\ell$. For each of the above training objectives, we train five different classifiers, each with a different random seed and thus a different 80\%-20\% train-test split. We use Adam~\cite{kingma2014adam} as the optimizer with a learning rate of $10^{-3}$ and a batch size of 100. To determine a suitable number of training epochs for each dataset and training objective, we train for 500 epochs and select the number of training epochs which leads to the best predictive performance in terms of accuracy and Mathews Correlation Coefficient (MCC). The resulting number of training epochs used are presented in Table~\ref{tb:epochs}. 

A tacit assumption in most of the algorithmic recourse literature is that the decision threshold for classifiers of the form $h(x) = \tilde h(x) \geq b$ for $\tilde h: \mathcal{X}\rightarrow\left(0,1\right)$ is set to $b=0.5$. In practical applications, however, the decision threshold $b$ is most often selected based on the nature of the classification problem at hand. For instance, in loan application settings the banking institution has a particular interest in minimizing false positives, since individuals which are mistakenly deemed low risk are very likely to result in monetary losses for the bank. We make the simplifying choice of, after training the classifier $\tilde h (x)$, choosing the decision threshold $b\in(0, 1)$ which maximizes the classifier's MCC. A classifier's MCC is a good metric to describe the classifier's confusion matrix with a single number, where higher MCC are associated with better predictive performance~\cite{powers2020evaluation}.

For the ALLR and \citet{ross2021learning} regularizers, we perform hyperparameter search over the hyperparameter $\mu$. We choose the regularization strength $\mu$ which most facilitate the existence of robust recourse while maintaining sufficient predictive performance. For the regularizer of Ross et al., we search over $\mu\in\{0.01, 0.1, 0.8, 1.5\}$ and find $\mu=0.8$ to work best across all datasets and model types. This aligns with the findings of \citet{ross2021learning}. For ALLR with NN classifiers, we heuristically find that $\mu_1=3.0$ works well across all datasets. We additionally perform hyperparameter search over $\mu_2\in\{0.01, 0.1, 0.5, 3.0\}$.  The best-performing hyperparameters $\mu_2$ are presented in Table~\ref{tb:lambdas}.

\begin{table}[h]
\centering
\caption{Training epochs for each training method, dataset and model type}
\begin{tabular}{c c c c c c c c c c c}
\toprule
\multirow{3}{*}{\makecell{Training \\ method}} & \multicolumn{10}{c}{Dataset \& model type} \\
 & \multicolumn{2}{c}{COMPAS} & \multicolumn{2}{c}{Adult} & \multicolumn{2}{c}{Loan} & \multicolumn{2}{c}{German} & \multicolumn{2}{c}{Bail} \\ 
   & LR & NN & LR & NN & LR & NN & LR & NN & LR & NN \\[1.ex] \hline \\[-1.5ex]
  ERM \& AF & 100 & 10 & 30 & 30 & 20 & 100 & 500 & 20 & 200 & 50 \\
  Ross et al. & 20 & 10 & 20 & 80 & 30 & 20 & 20 & 20 & 40 & 100 \\
  ALLR & 10 & 20 & 20 & 80 & 20 & 30 & 40 & 20 & 20 & 500 \\
  \bottomrule 
\end{tabular}
\label{tb:epochs}
\end{table}

\begin{table}[h]
\centering
\caption{Hyperparameter $\mu_2$ used for ALLR}
\begin{tabular}{c c c c c c}
\toprule
\multirow{2}{*}{Model type} & \multicolumn{5}{c}{Dataset} \\
& COMPAS & Adult & Loan & German & Bail \\[1.ex] \hline \\[-1.5ex]
LR & 0.1 & 0.1 & 0.1 & 0.1 & 0.1 \\
NN & 0.1 & 0.5 & 0.01 & 0.5 & 0.01 \\
 \bottomrule 
\end{tabular}
\label{tb:lambdas}
\end{table}

\subsection{Algorithm~1}
\subsubsection{Projecting to the uncertainty set}
 For $\epsilon$-neighborhoods $B(x) = \{\mathbb{CF}(x, \Delta) \; | \; \norm{\Delta} \leq \epsilon\}$, one need only project to the $\epsilon$-ball of the norm $\norm{\cdot}$, since  in general $\max_{x'\in B(x)} f(x) = \max_{\norm{\Delta} \leq \epsilon } \;  f(\mathbb{CF}(x, \Delta))$. For certain choices of norm $\norm{\cdot}$ (e.g., the $l_2$ norm), projecting to the $\epsilon$-ball can be efficiently done in closed form. 
 
\subsubsection{Hyperparameter tuning}
 We set $N_{\text{max}}=100$ according to our computational budget. To solve the inner maximization problem, we additionally allow a maximum of $50$ gradient steps, for a total of $100\cdot50=5000$ optimization steps. We tune $\lambda$ heuristically. If $\lambda$ is too large, few recourse actions are found (low weight given to crossing the decision boundary), whereas if $\lambda$ is too small, recourse actions tend to be overtly costly (low weight given to finding low-cost recourse). We find $\lambda=1.0$ to work well across all datasets. Additionally, we find $\gamma=0.9$ to work well, as $\lambda$ then decreases relatively slowly (favoring low-cost recourse being found) but $\lambda$ is nonetheless close to 0 after $N_{\text{max}}$ iterations (favoring crossing the decision boundary).

\subsection{Metrics considered}

\begin{itemize}[topsep=0pt]
\setlength\itemsep{0em}
    \item Magnitude of min. perturbation invalidating recourse: for some classifier $h$, individual $x$ and recourse action $a$, we use the C\&W adversarial attack~\cite{carlini2017towards} to search for the minimum additive intervention $\Delta_{\text{adv}}$ which invalidates $a$, that is, $\Delta_{\text{adv}} = \argmin_{\Delta} \norm{\Delta}_2 \text{ s.t. } h(\mathbb{CF}(\mathbb{CF}(x, \Delta), a)) = 0$. We report its magnitude $\norm{\Delta_{\text{adv}}}_2$. While we tested a variety of adversarial attacks, including the fast gradient sign method~\cite{goodfellow2015explaining}, projected gradient descent~\cite{madry2017towards} and DeepFool~\cite{moosavi2016deepfool}, we found C\&W to work best. 
    \item \% recourse found: for some classifier $h$ and uncertainty $\epsilon$, we sample 1000 negatively classified individuals from the test set and use our proposed methods to generate recourse actions which are robust against $\epsilon$ uncertainty. Out of those individuals for which a valid recourse action is found, we then use the C\&W attack described above to generate adversarial perturbations $\Delta_{\text{adv}}$. We report the rate of individuals for which $\norm{\Delta_\text{adv}}_2  \geq \epsilon$ (i.e., for which some recourse action is found and we are not able to invalidate such action) out of the original 1000 individuals.
    \item Relative cost of recourse: we aim to assess whether classifiers trained using the proposed model regularizer offer (robust) recourse at a higher or lower cost compared to classifiers trained using the standard ERM approach. For any given model regularizer, we consider the individuals for which valid (robust) recourse is found for both the ERM classifier and the regularizer classifier. We then report the mean cost of recourse offered to those individuals under the regularized decision-making classifiers. To improve clarity of presentation, for each dataset we normalize the mean cost of recourse found such that the maximum mean cost of recourse reported is 
    1.
\end{itemize}

\section{Ablation study for the ALLR regularizer}

The proposed ALLR regularizer comprises two penalty terms, one which encourages the decision-making classifier to behave locally linearly, and a second one that penalizes locally relying on unactionable features. The two penalty terms are weighted by the hyperparameters $\mu_1$ and $\mu_2$, respectively. To verify that both penalty terms are necessary to facilitate the existence of robust recourse, we use the ALLR hyperparameters presented in Table~\ref{tb:lambdas} to train classifiers with two additional regularization approaches related to ALLR: one for which $\mu_1 = 0$ and a second one for which $\mu_2 = 0$ (i.e., only one of the two penalty terms is considered at a time). We only consider NN classifiers, since for linear classifiers the first penalty term (to behave locally linearly) is trivially always 0. We present results in Figure~\ref{fig:ablation}. We observe that both penalty terms are necessary to facilitate the existence of robust recourse.

 \begin{figure*}[h]
      \begin{center}
    \scalebox{0.85}{\begin{adjustbox}{clip,trim=0.0cm 0.0cm 0.0cm 0.0cm}
            \includegraphics{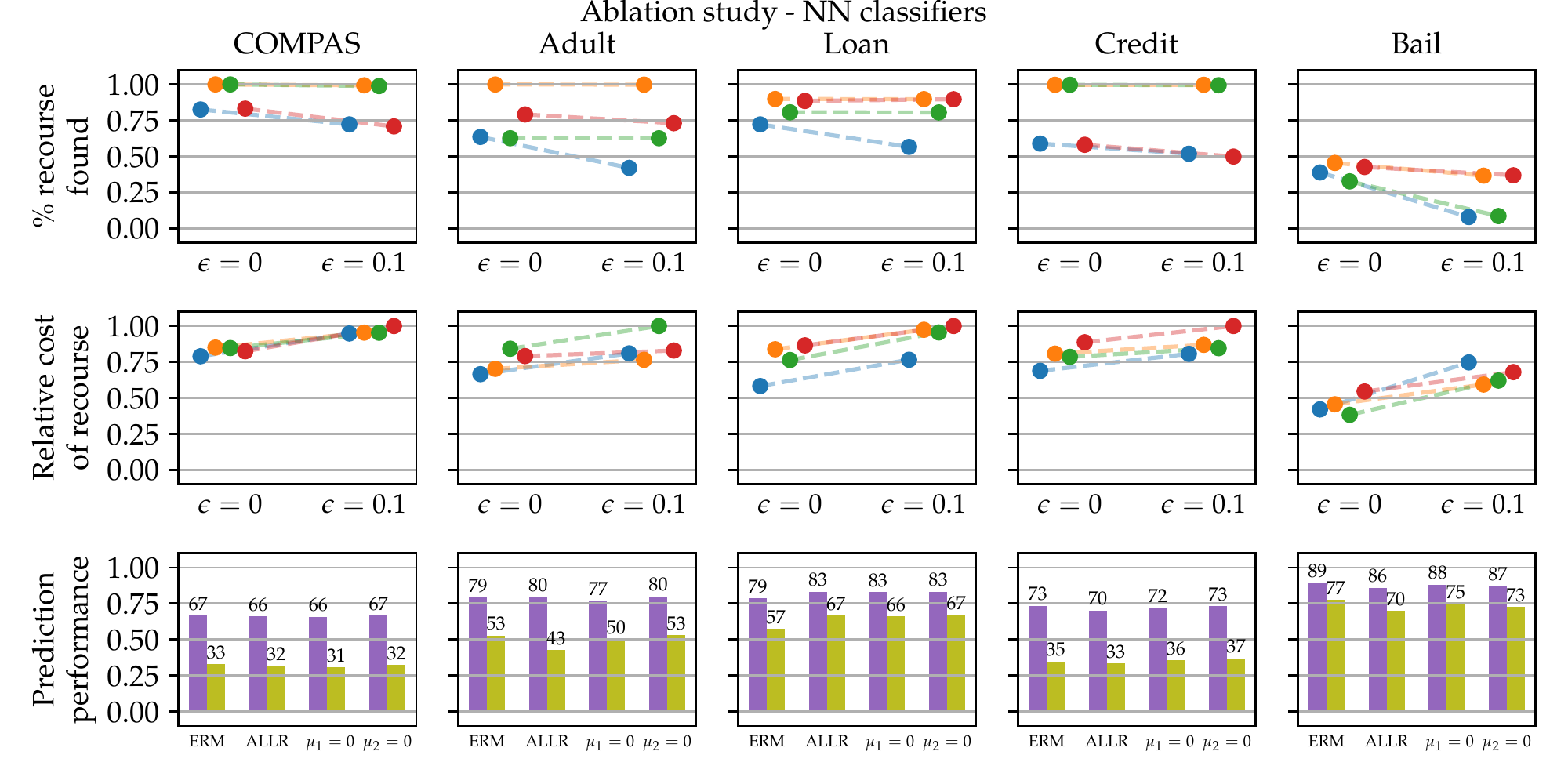}
    \end{adjustbox}}
    \end{center}
    \caption{Results for the ablation study. Both penalties in ALLR are important to facilitate the existence of adversarially robust recourse. Legend: \legendsquare{def1} ERM~~\legendsquare{def2}~ALLR~~\legendsquare{def3}~ALLR $\mu_1=0$~~\legendsquare{def4}~ALLR $\mu_2=0$~~\legendsquare{def5}~Accuracy~~\legendsquare{def6}~MCC score.}
        \label{fig:ablation}
\end{figure*}

\end{document}